\documentclass{article}

\usepackage[preprint]{style}

\usepackage[T1]{fontenc}
\usepackage[utf8]{inputenc}
\usepackage{microtype}

\usepackage{amsmath,amssymb,amsfonts}
\usepackage{amsthm}
\usepackage{booktabs}
\usepackage[final]{graphicx}
\usepackage{xcolor}
\usepackage{algorithm}
\usepackage{algpseudocode}
\usepackage{caption}
\usepackage{subcaption}
\usepackage{wrapfig}
\usepackage{mdframed}
\usepackage[colorlinks=true,allcolors=blue]{hyperref}
\usepackage[capitalise,noabbrev]{cleveref}

\numberwithin{equation}{section}
\numberwithin{figure}{section}
\numberwithin{table}{section}

\theoremstyle{plain}
\newtheorem{theorem}{Theorem}[section]
\newtheorem{lemma}{Lemma}[section]
\theoremstyle{remark}

\newtheoremstyle{boxedplain}%
  {0pt}%
  {0pt}%
  {\itshape}%
  {}%
  {\normalfont}%
  {.}%
  {0.5em}%
  {}%
\theoremstyle{boxedplain}
\newtheorem{boxedtheorem}[theorem]{Theorem}
\theoremstyle{remark}

\newmdenv[
  linewidth=0.45pt,
  linecolor=black!38,
  backgroundcolor=black!2,
  roundcorner=4pt,
  innerleftmargin=10pt,
  innerrightmargin=10pt,
  innertopmargin=8pt,
  innerbottommargin=8pt,
  skipabove=10pt,
  skipbelow=10pt
]{theorembox}

\newcommand{\erank}{\operatorname{erank}}
\newcommand{\tr}{\operatorname{tr}}
\newcommand{\diag}{\operatorname{diag}}
\newcommand{\Rbb}{\mathbb{R}}
\newcommand{\ip}[2]{\left\langle #1, #2 \right\rangle}
\newcommand{\Fnorm}[1]{\left\|#1\right\|_F}
\newcommand{\Snorm}[1]{\left\|#1\right\|_2}
\newcommand{\Nnorm}[1]{\left\|#1\right\|_*}

\title{ORCA: The Annealed Spectral Conditioning\\
Optimizer for Faster, Better LLM Training}

\author{%
  {\small
    Yuanshi Liu\textsuperscript{1}\thanks{This work was conducted during the author’s internship at Kling Team.}
    \hspace{0.34em}Boyuan Jiang\textsuperscript{2}
    \hspace{0.34em}Liang Hou\textsuperscript{2}
    \hspace{0.34em}Xin Tao\textsuperscript{2}
    \hspace{0.34em}Pengfei Wan\textsuperscript{2}
    \hspace{0.34em}Zhouchen Lin\textsuperscript{1}
    \hspace{0.34em}Cong Fang\textsuperscript{1}
  }
  \\[2pt]
  \normalfont\small\itshape
  \textsuperscript{1}School of Intelligence Science and Technology, Peking University
  \quad \textsuperscript{2}Kling Team
}

\begin{document}

\maketitle

\begin{abstract}
  Modern LLM optimizers such as Muon often produce weight matrices with higher effective rank than Adam, yet further spectral control has delivered only modest gains. We identify a tension behind this result: concentrated spectra can suppress gradient directions in coupled weight matrices and slow optimization, while constraints maintained throughout training can limit task-specific adaptation and raise the attainable loss floor. We introduce ORCA (Orthogonal Regularization, Cooled After), a minimal optimizer intervention that applies strong but temporary soft orthogonality regularization early in training, then removes it. This allows the weights to benefit from a broader spectrum early on and adapt freely afterward.
  Across LLaMA, Qwen3, and fine-grained mixture-of-experts models ranging from 130M to 8B parameters, ORCA achieves lower final validation loss than Muon. Its loss reduction relative to Muon matches or exceeds Muon's reduction relative to Adam. Ablations support the early-shaping, later-release design. Further, ORCA requires no architectural changes and adds minimal overhead. An implementation is available at \url{https://github.com/uuurraaa/ORCA_Optimizer}.

\end{abstract}

\section{Introduction}
\label{sec:intro}

Recent matrix-based optimizers such as Muon have shown substantial improvements over Adam in LLM training \citep{liu2025muon, jordan2024muon}. A striking difference emerges in their weight trajectories: compared with Adam, Muon maintains a markedly flatter singular-value spectrum and a substantially higher effective rank over most of training, where effective rank is measured as $\|\mathbf{W} \|^2_F / \| \mathbf{W}\|^2_2$. The two optimizers start from comparable values, but Adam exhibits a much sharper decline early in training and remains consistently below Muon thereafter. This observation suggests a natural hypothesis: better preserving the effective dimensionality of weight matrices may be an important ingredient behind faster optimization.

This emerging importance of spectral structure has motivated an important and growing line of work that explicitly preserves or promotes favorable spectral structure during training \citep{qiu2026reparameterized, shi2026pion, dong2026muon}. These approaches introduce principled mechanisms for controlling the weight spectrum and have demonstrated promising improvements in LLM optimization. If a higher effective rank is indeed beneficial, however, one might expect stronger spectral control to yield correspondingly larger gains. In practice, the improvements over Muon are often more modest than this intuition would suggest, despite substantially stronger control over the weight spectrum.

To understand this gap, we examine two competing effects of spectral control that arise simultaneously. We begin with a matrix factorization model \citep{gunasekar2017implicit} that abstracts the multiplicative coupling common in Transformer architectures. On the benefit side, we illustrate that a well-shaped spectrum allows more gradient directions to contribute to other parameter updates, facilitating faster optimization. On the cost side, however, imposing spectral constraints limits model expressiveness and can raise the attainable loss floor. Correspondingly, controlled interventions during Llama training empirically support both effects. Together, these results point to a seemingly trade-off between faster optimization and model expressiveness.

However, we argue that this trade-off is not fundamental. Spectral regularization may be most valuable early, when it can shape the optimization trajectory, yet become a liability if maintained as the weights adapt to the task. The two effects need not be accepted together: we can use the constraint to shape the spectrum, then release it before it restricts the final solution.
This motivates a \textbf{minimal} method ORCA (Orthogonal Regularization, Cooled After): During an initial phase, ORCA applies a soft orthogonality penalty, \(R_{\mathrm{orth}}(\mathbf W)=\frac14\|\mathbf W\mathbf W^\top-\mathbf I\|_F^2\). Then directly removes it at a single switch point. Training thereafter follows the standard optimizer update, with no gradual decay or further changes. This design lets early regularization influence the optimization path without prescribing the spectrum throughout training. In our experiments, this simple release is followed by faster loss reduction and a lower final validation loss than Muon.

\begin{wrapfigure}[20]{r}{0.43\linewidth}
  \vspace{-0.8\baselineskip}
  \centering
  \includegraphics[width=\linewidth]{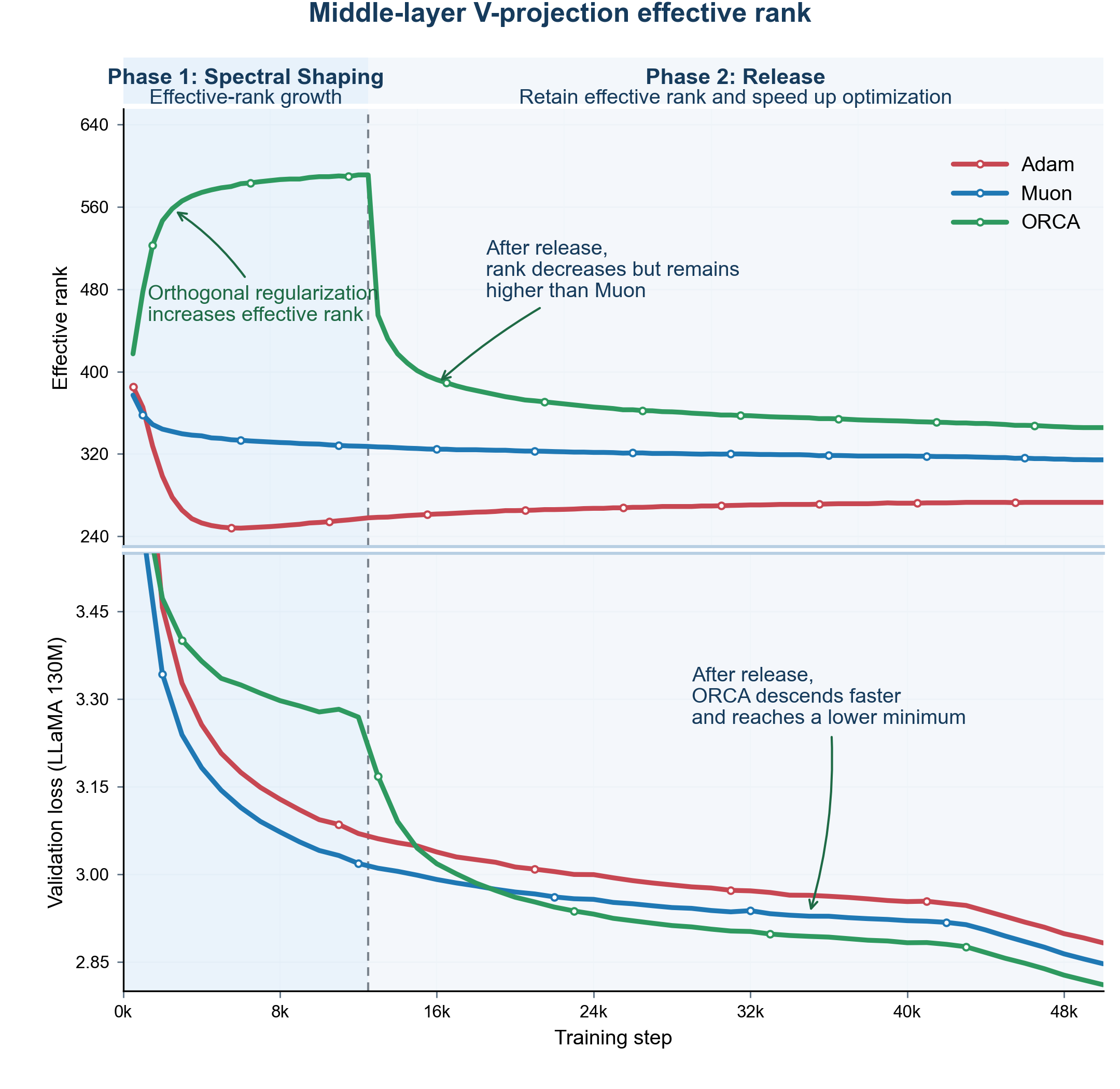}
  \captionsetup{skip=2pt}
  \caption{LLaMA 130M training dynamics. Effective rank is measured on a
  middle-layer $V$ projection; the dashed line marks regularization release.}
  \label{fig:intro-rank-loss}
\end{wrapfigure}

Figure~\ref{fig:intro-rank-loss} illustrates the transition. While the penalty is active, the effective rank of the weight matrix rises sharply above Muon's, although ORCA initially has a higher validation loss. After release, the rank falls from its peak but remains above Muon's, while validation loss decreases rapidly, crosses the Muon curve. The result is consistent with early spectral shaping leaving a useful effect on the subsequent trajectory, even after the explicit penalty is gone. Controlled interventions in Section~\ref{sec:interventions} further investigates into this connection.

We evaluate ORCA on LLaMA, Qwen3, and fine-grained MoE models ranging from 130M to 8B parameters, including 100B-token training runs. ORCA achieves lower final validation loss than Muon across these settings. In most experiments, the reduction is 0.02–0.04, comparable to or greater than Muon’s gain over Adam. To put this loss reduction in perspective, on LLaMA 1.3B, ORCA achieving the 0.024 lower validation loss reaches Muon’s 100k-step loss in just 60k steps. Further, downstream evaluations, alternative early spectral interventions, multi-seed runs, and ablations of method parameters and intervention timing further support these gains. Together, these results show that early spectral shaping followed by constraint relaxation can improve training efficiency and final performance across architectures without substantial changes to the optimizer.

To summarize the contributions of the paper: (1) We illustrate a trade-off in spectral control through a simplified model and empirical observations: favorable spectra can accelerate optimization but persistent constraints may harm the attainable solution quality. (2) We introduce ORCA, a minimal shape-then-release optimizer that leverages early orthogonal regularization to improve optimization trajectories while allowing unconstrained adaptation in later training. (3) We demonstrate ORCA's effectiveness across diverse LLM architectures and scales, achieving faster convergence and lower loss than existing optimizers. Through extensive controlled experiments, we further validate the mechanism that early effective-rank shaping drives the observed improvements.

\section{Spectral Structure and Optimization}

As noted in Section 1, Muon’s higher effective rank coincides with faster optimization. This observation suggests that properties of the weight spectrum may be useful for optimization, but whether they provide a general benefit remains unclear. To turn this still-vague possibility into a precise question about optimization dynamics, we ask:
\begin{center}
\fbox{
\parbox{0.92\linewidth}{
\centering
\textbf{Q1.}
\emph{Can the spectral structure observed under Muon directly affect
optimization speed, and through what mechanism?}
}}
\end{center}

A direct mechanism already appears in 'the multiplicatively coupled weight
matrices' of Transformers. Consider for example the value--output pair $\mathbf W_V\mathbf W_O$ of a single attention
head:
\begin{equation}
\operatorname{Attn}(\mathbf H)
=\operatorname{softmax}\!\left(
\frac{\mathbf H\mathbf W_Q\mathbf W_K^\top\mathbf H^\top}{\sqrt{d_h}}
\right)\mathbf H\mathbf W_V\mathbf W_O .
\end{equation}
The gradient of the value projection takes the form
\[
\nabla_{\mathbf W_V}L
=
\mathbf G\mathbf W_O^\top,
\quad \text{with} \quad
\mathbf G=\nabla_{(\mathbf W_V\mathbf W_O)}L .
\]
Here, $\mathbf G$ denotes the gradient with respect to the entire product
$\mathbf W_V\mathbf W_O$.
Thus, $\mathbf W_O$ acts as a spectral filter on the gradient received by
$\mathbf W_V$. Writing
$\mathbf W_O=\sum_i \sigma_i\mathbf u_i\mathbf v_i^\top$, we have
$\nabla_{\mathbf W_V}L
=\sum_i \sigma_i(\mathbf G\mathbf v_i)\mathbf u_i^\top$.
This makes the role of the spectrum explicit: each component
$\mathbf G\mathbf v_i$ enters the gradient with weight $\sigma_i$.
If $\mathbf W_O$ is scaled orthogonal,
$\mathbf G\mathbf W_O^\top$ preserves the spectrum of $\mathbf G$ up to a
common scale. At the opposite extreme, if $\mathbf W_O$ has rank $r$,
the update is restricted to contributions from $r$ spectral directions.
More generally, as the spectrum becomes more concentrated, directions
associated with smaller singular values contribute proportionally less to
the update.

In fact, this partner-dependent filtering is ubiquitous in Transformers:
analogous coupling appears between the query--key pair
$(\mathbf W_Q,\mathbf W_K)$ and the up- and down-projections of MLP and expert
blocks.

To isolate this mechanism from the surrounding Transformer computation, consider the generic factorization $\mathbf W=\mathbf X\mathbf Y$, where $\mathbf G=\nabla_{\mathbf W}L(\mathbf W)$.
The simple chain rule gives
\(
\nabla_{\mathbf X}L
=\mathbf G\mathbf Y^\top,
\nabla_{\mathbf Y}L
=\mathbf X^\top\mathbf G,
\)
so each factor directly filters the gradient received by its partner. Under the
continuous-time exact-polar Muon dynamics \citep{ambrosio2005gradient,peyre2026muon},
\[
\dot{\mathbf X}
=-\eta_t^X\operatorname{Polar}(\mathbf G\mathbf Y^\top),
\qquad
\dot{\mathbf Y}
=-\eta_t^Y\operatorname{Polar}(\mathbf X^\top\mathbf G),
\]
where $\operatorname{Polar}(\mathbf M)=\mathbf U\mathbf V^\top$
for an SVD $\mathbf M=\mathbf U\mathbf\Sigma\mathbf V^\top$.
The instantaneous loss decrease is therefore
\[
-\dot L
=\eta_t^X\Nnorm{\mathbf G\mathbf Y^\top}
+
\eta_t^Y\Nnorm{\mathbf X^\top\mathbf G} .
\]
Thus, the same filtered gradients above are the quantities that determine Muon's instantaneous loss decrease. The local
spectral effect can therefore translate into a local effect on optimization speed.

To empirically demonstrate this effect, we conduct a reciprocal $V$--$O$ intervention.
Let $\mathbf G_O=\nabla_{\mathbf W_O}L$ denote the gradient with respect to
$\mathbf W_O$, and let $\mathbf U_O=\operatorname{Polar}(\mathbf G_O)$ denote
the corresponding Muon update direction.
At a 10k-step checkpoint, we raise the smaller singular values of $\mathbf W_V$
toward its largest singular value while keeping its singular vectors and spectral
norm fixed, and update only $\mathbf W_O$ with Muon.
As shown in \Cref{fig:v-rank-o-metrics}, with the Muon update scale
$\|\mathbf U_O\|_F$ held fixed, increasing the effective rank of $\mathbf W_V$
increases $\|\mathbf G_O\|_*$---the loss decrease under Muon
flow---while decreasing the gradient--weight alignment
$\cos(\mathbf G_O,\mathbf W_O)$.
These all indicate that the same-sized Muon update has a stronger effect as the effective rank of
$\mathbf W_V$ increases.
Consistently, \Cref{fig:v-rank-o-surface} also shows that updating $\mathbf W_O$
yields a larger held-out loss decrease at higher effective rank of $\mathbf W_V$.

\begin{figure}[htb]
  \centering
  \captionsetup{font=small,skip=4pt}
  \begin{subfigure}[t]{0.32\linewidth}
    \centering
    \includegraphics[width=\linewidth]{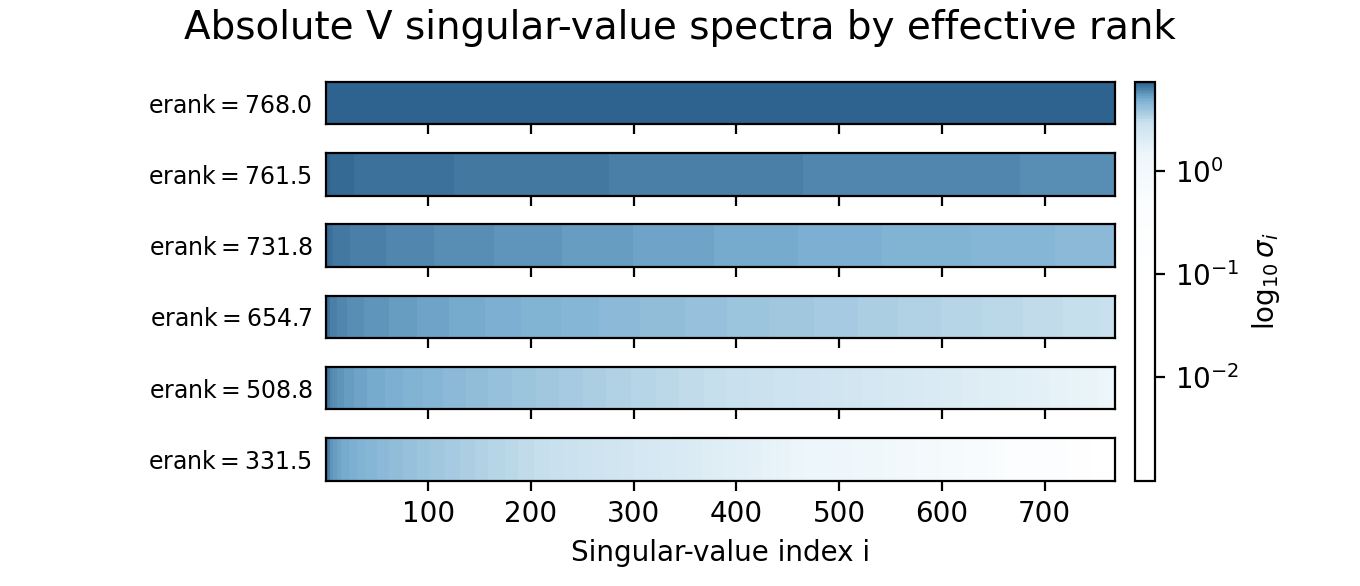}
    \caption{Absolute singular-value spectra.}
    \label{fig:v-rank-o-spectrum}
  \end{subfigure}\hfill
  \begin{subfigure}[t]{0.32\linewidth}
    \centering
    \includegraphics[width=\linewidth]{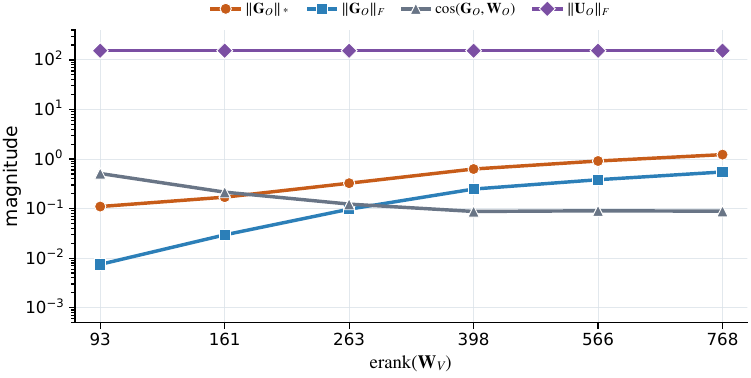}
    \caption{Gradient and update metrics.}
    \label{fig:v-rank-o-metrics}
  \end{subfigure}\hfill
  \begin{subfigure}[t]{0.32\linewidth}
    \centering
    \includegraphics[trim=0 0 0 12,clip,width=\linewidth]{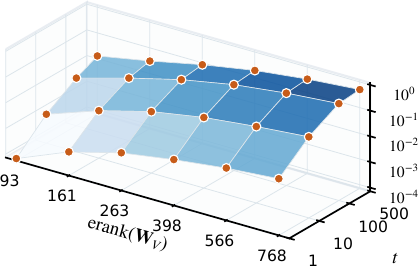}
    \caption{Held-out loss decrease.}
    \label{fig:v-rank-o-surface}
  \end{subfigure}
  \captionsetup{skip=2pt}
  \caption{Single-layer $V$-spectrum intervention with $O$-only Muon updates.
  $\mathbf G_O$ and $\mathbf U_O$ denote the gradient and Muon update direction.
  In (c), rollout steps and loss decrease use logarithmic scales.}
  \label{fig:v-rank-o-diagnostics}
\end{figure}

These observations suggest a simple principle: spectral concentration can create an optimization bottleneck. For a nonzero matrix $\mathbf A$, define its effective rank by $\erank(\mathbf A):=\Fnorm{\mathbf A}^2/\Snorm{\mathbf A}^2$.
To make this intuition concrete, we study a matrix-factorization model. The following bounds illustrate one possible mathematical pathway through which spectral structure may affect optimization.
\vspace{-2pt}
\begin{theorembox}
\begin{boxedtheorem}[Low Effective Rank Limits Descent Speed]
\label{thm:mechanistic}
For $k\ge2$, let $\mathbf X$ and $\mathbf Y$ be conformable matrices with
$\mathbf W=\mathbf X\mathbf Y\in\mathbb R^{k\times k}$, and consider the
column-wise softmax loss
\begin{equation}
 L(\mathbf W)
 =\sum_{i=1}^{k}
 \left[
 \log\!\left(\sum_{j=1}^{k}e^{\mathbf W_{ji}}\right)
 -\mathbf W_{c_i i}
 \right],
\quad c_i\in[k], \quad \text{with} \quad \mathbf{W} = \mathbf{X}\mathbf{Y}
\label{eq:objective}
\end{equation}
let $(\mathbf X_t,\mathbf Y_t)_{t\in[0,T]}$ be an absolutely continuous solution
of the exact-polar Muon flow above, with $\eta_0\ge0$ and
\(
0\le\eta_t^X\le\eta_0/\Snorm{\mathbf Y_t},
\qquad
0\le\eta_t^Y\le\eta_0/\Snorm{\mathbf X_t}.
\)
Assume that the factor norms satisfy
\(
\Snorm{\mathbf X_t},\Snorm{\mathbf Y_t}>0
\)
for all $t\in[0,T]$.
Writing $L_t=L(\mathbf X_t\mathbf Y_t)$, we have
\[
L_T
\ge
 L_0\exp\!\left[
-\sqrt{2}\,\eta_0
\int_0^T
\left(
\sqrt{\erank(\mathbf X_s)}
+
\sqrt{\erank(\mathbf Y_s)}
\right)ds
\right].
\]
\end{boxedtheorem}
\end{theorembox}

\begin{wrapfigure}{r}{0.34\linewidth}
  \centering
  \captionsetup{font=footnotesize,skip=3pt}
  \includegraphics[width=\linewidth]{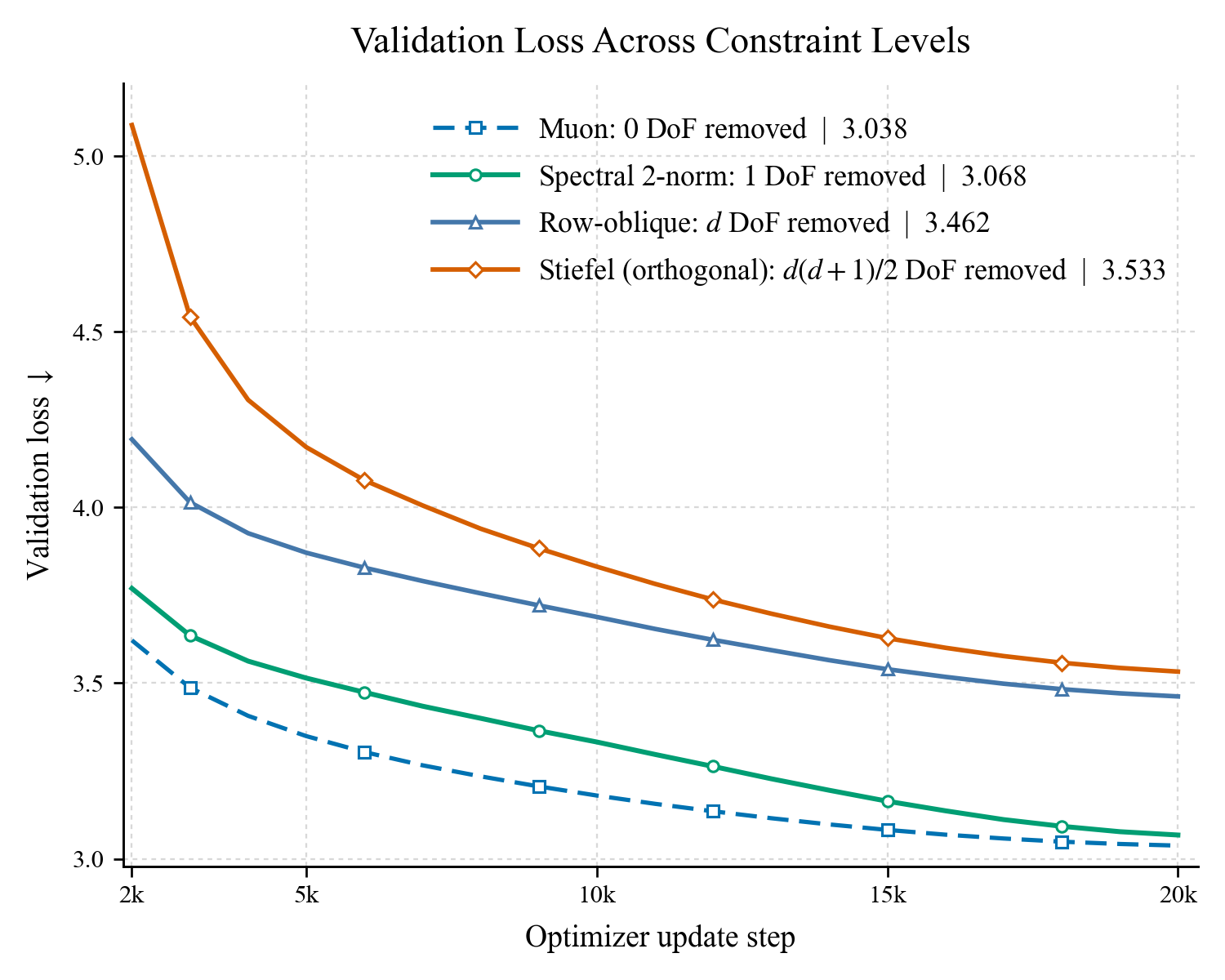}
  \captionsetup{skip=2pt}
  \caption{Validation loss under persistent spectral constraints.}
  \label{fig:constraint_levels}
\end{wrapfigure}
\vspace{-2pt}
In this factorization model, the bound relates the rate of loss decrease to the effective ranks of the factors. Naturally, broad spectral support may alleviate this bottleneck, as demonstrated by the intervention experiments in \Cref{fig:v-rank-o-diagnostics}. We therefore test whether maintaining such support throughout LLaMA pretraining helps by imposing progressively stronger spectral constraints. However, \Cref{fig:constraint_levels} shows that persistent orthogonal constraints worsen validation loss.
Thus, the path-side advantage identified above does not translate into a benefit
when the same structure is enforced persistently. This leads to our second
question:
\vspace{-2pt}
\begin{center}
\fbox{
\parbox{0.92\linewidth}{
\centering
\textbf{Q2.}
\emph{If broad spectral structure alleviates an optimization bottleneck,
why enforcing it hurts performance?}
}}
\end{center}
\makeatletter
\par
\WFclear
\ifdim\linewidth<\hsize
  \advance\linewidth\WF@adjlw
  \advance\@totalleftmargin-\WF@adjtlm
\fi
\makeatother
\vspace{-3pt}
The key is that a property that is favorable along the optimization trajectory
need not be favorable when imposed as a persistent constraint.
Spectral control does more than influence how an optimizer moves: when enforced
throughout training, it also restricts the spectral structure of the solutions
that the optimizer is allowed to reach.
For example, an orthogonality constraint
\(
\mathbf{W}\mathbf{W}^\top\approx\mathbf{I}
\)
forces the singular values of $\mathbf{W}$ to remain nearly equal.
If the task prefers a more anisotropic or lower-effective-rank solution, such a
constraint can prevent the optimizer from reaching it.

The factorized model above makes this competing effect concrete.
For the same column-wise softmax objective in
Eq.~\eqref{eq:objective},
the following theorem shows that sufficiently high effective rank can itself
impose a nonzero loss floor.

\begin{theorembox}
\begin{boxedtheorem}[High Effective Rank Can Impose a Loss Floor]
\label{thm:endpoint}
Assume $k\ge2$, $\mathbf W\neq0$, and
\(\|\mathbf W\|_2\le B\) for a fixed scale bound $B=O(1)$. Let
\(
K
=
\left|
\{c_i:i\in[k]\}
\right|
\)
denote the number of distinct target indices. Then
\[
L(\mathbf{W})
\ge
C_k(\mathbf{W})\,k
\left(
1-
\sqrt{
\frac{K}{\erank(\mathbf{W})}
}
\right),
\]
where $C_k(\mathbf{W})>0$ is the scale factor determined by $k$ and
$\|\mathbf{W}\|_F/\sqrt{k}$. The rank factor
$1-\sqrt{K/\erank(\mathbf{W})}$ increases with effective rank and becomes
positive when $\erank(\mathbf{W})>K$.
\end{boxedtheorem}
\end{theorembox}

\vspace{-5pt}

Thus, for this objective, sufficiently high effective rank is not merely
unnecessary at the optimum: it can be incompatible with arbitrarily small
loss. In particular, if $\mathbf{X}$ and $\mathbf{Y}$ are $k\times k$ scaled
orthogonal matrices, then
$\erank(\mathbf{X}\mathbf{Y})=k$, and Theorem~\ref{thm:endpoint} gives the
positive endpoint loss floor
\(
C_k(\mathbf{W})\,k
\left(
1-\sqrt{K/k}
\right)
\)
whenever $K<k$.

Together, these results illustrate a tension in the stylized model: broad spectral support can relax a path-side bottleneck, while a persistent spectral constraint can restrict low-loss solutions. We also test whether these intuitions are useful in the empirical training.
The same spectral property can therefore be useful as a transient feature of
the optimization trajectory while being harmful as a persistent constraint on
the final solution.

This distinction helps reconcile the empirical behavior above.
The high effective rank naturally maintained by Muon can support optimization
without prescribing a fixed spectral geometry for the final weights.
By contrast, increasingly restrictive spectral constraints preserve the desired
structure by reducing the freedom of the model to adapt its spectrum as
training proceeds.
The appropriate goal is therefore not to maximize effective rank at every
point in training, but to retain enough spectral support to avoid path-side
bottlenecks while still allowing the spectrum to adapt to the task.

\section{ORCA: A Minimal Shape-then-Release Optimizer}
\label{sec:orca}

The preceding analysis reveals a tension between the trajectory-side benefit
and the endpoint-side cost of spectral control.
Broad spectral support can facilitate optimization along the training
trajectory, while enforcing the same structure persistently may restrict
task-adaptive solutions at the endpoint.
This suggests a temporal design principle: spectral control should shape the
early optimization trajectory, but be released before it restricts task
adaptation.

ORCA (Orthogonal Regularization, Cooled After) is a minimal two-stage
realization of this shape-then-release principle.
During an initial active stage, ORCA applies spectral shaping to steer the
early trajectory away from spectrally concentrated regimes.
After this stage, the shaping term is removed completely and training reverts
to standard Muon.
This separation allows spectral control to influence the optimization path
without prescribing the spectral geometry of the final solution.

\subsection{Algorithm}
\label{sec:orca-algorithm}

\noindent\textbf{A minimal realization: soft orthogonal shaping.}
For this minimal realization, ORCA uses soft orthogonality as its spectral
shaping operator.
Consider a matrix parameter $\mathbf{W} \in \mathbb{R}^{p \times q}$, and assume
$p \leq q$ for notational simplicity.
We define
$
    \mathcal{R}_{\mathrm{orth}}(\mathbf{W})
    =
    1/4\
    \left\|\mathbf{W}\mathbf{W}^\top-I_p\right\|_F^2 ,
    \label{eq:orca-orth-reg}
$
whose gradient is
\begin{equation}
    G_{\mathrm{orth}}(\mathbf{W})
    :=
    \nabla_{\mathbf{W}}\mathcal{R}_{\mathrm{orth}}(\mathbf{W})
    =
    \left(\mathbf{W}\mathbf{W}^\top-I_p\right)W .
    \label{eq:orca-orth-grad}
\end{equation}

Writing a thin singular value decomposition as
$\mathbf{W}=\mathbf{U}\operatorname{diag}(\sigma_1,\ldots,\sigma_p)\mathbf{V}^\top$, we have
\begin{equation}
\begin{aligned}
    \mathcal{R}_{\mathrm{orth}}(\mathbf{W})
    =
    \frac{1}{4}
    \sum_{i=1}^{p}
    \left(\sigma_i^2-1\right)^2, \quad
    \mathbf{G}_{\mathrm{orth}}(\mathbf{W})
    =
    \mathbf{U}\operatorname{diag}
    \left(
        (\sigma_i^2-1)\sigma_i
    \right)_{i=1}^{p}
    \mathbf{V}^\top .
\end{aligned}
\label{eq:orca-orth-svd}
\end{equation}
Thus, a gradient step on $\mathcal{R}_{\mathrm{orth}}$ decreases singular
values above unit scale and increases nonzero singular values below unit scale,
tending to reduce spectral concentration and increase the effective-rank
measure considered in the preceding analysis.
The operator requires neither an explicit SVD nor a projection onto the
Stiefel manifold and can be computed using matrix multiplications alone.

\noindent\textbf{ORCA update.}
Let $\mathbf{U}_t$ denote the standard Muon update direction at step $t$,
$\eta_t$ the learning rate, $\lambda$ the decoupled weight-decay coefficient,
and $\gamma_t$ the spectral-shaping strength.
For each matrix parameter $\mathbf{W}_t$, the ORCA update can be written as
\begin{equation}
    \mathbf{W}_{t+1}
    =
    \mathbf{W}_t
    -
    \eta_t \mathbf{U}_t
    -
    \eta_t\lambda \mathbf{W}_t
    -
    \eta_t\gamma_t \mathbf{G}_{\mathrm{orth}}(\mathbf{W}_t).
    \label{eq:orca-update}
\end{equation}
The weight-decay and spectral-shaping corrections are both evaluated at the
pre-update parameter $\mathbf{W}_t$ \citep{loshchilov2017decoupled}, while $\mathbf{U}_t$ is computed in the same way as in
standard Muon.

Similar to Muon and other matrix-based optimizers, ORCA is applied only to
matrix parameters in the matrix-optimizer group.
Embeddings, the language-model head, and one-dimensional parameters continue
to follow their original AdamW updates.

\noindent\textbf{Release rule.}
The defining feature of ORCA is that spectral shaping is transient rather than
persistent.
Let $T$ denote the total number of training steps and
$\tau\in[0,1]$ the active fraction.
The active window has length $S$, and we use the hard release schedule
\begin{equation}
  \gamma_t =
  \begin{cases}
  \gamma, & t < S, \\
  0,      & t \geq S,
  \end{cases}
  \qquad
  S = \left\lfloor \tau T \right\rfloor.
  \label{eq:orca-gamma-schedule}
\end{equation}
Equation~\eqref{eq:orca-gamma-schedule} makes the two phases explicit.
During the \emph{active stage} ($t<S$), ORCA augments Muon with the
spectral-shaping term in Eq.~\eqref{eq:orca-update}.
During the \emph{released stage} ($t\geq S$), the shaping term vanishes and
the optimizer reverts exactly to the corresponding standard Muon training
rule.

The hard switch is the simplest implementation of the shape-then-release
principle.
The early intervention is used to steer the optimization trajectory away from
spectrally concentrated regimes, rather than to prescribe a persistent
spectral constraint.
Once released, the spectrum is free to adapt to the task objective.
Empirically, we find that part of the higher-effective-rank state induced
during the active stage persists after release, allowing the optimization
benefit to continue without maintaining the orthogonal regularizer.
Unless otherwise specified, we use $\tau=0.25$ throughout the paper;
Section~5 studies sensitivity to the release point.

\noindent\textbf{Beyond soft orthogonality.}
Soft orthogonality defines the default ORCA optimizer, but the underlying
shape-then-release principle is not tied to this particular spectral operator.
In Section~5, we replace soft orthogonal shaping with alternative early
spectral interventions, including Pion-start and DiagSVD variants.
These methods differ in how the spectral state is shaped during the active
stage; once the active window ends, the intervention is removed and all
variants continue with the same released Muon dynamics.
This provides a controlled way to study whether the benefit of early spectral
intervention extends beyond soft orthogonality itself.

\noindent\textbf{Implementation and overhead.}
In the standard setting, ORCA adds two matrix multiplications per weight matrix during shaping without changing the communication complexity; the remaining operations are fused into the optimizer update. Our Muon baseline performs 24 matrix multiplications across eight Newton--Schulz iterations, so shaping increases the multiplication count by \(2/24=8.3\%\). This extra work disappears after release. We batch matrices of identical shapes to reduce kernel launches, use dedicated kernels, and form the Gram matrix along the shorter dimension. On the 8B MoE model, average distributed optimizer update time is \(227.2\,\mathrm{ms}\) for Muon and \(237.4\,\mathrm{ms}\) for ORCA. The \(10.2\,\mathrm{ms}\) difference amounts to approximately \(0.8\%\) of a roughly \(1.3\,\mathrm{s}\) training step. Peak memory is \(101.41\,\mathrm{GB}\) and \(101.66\,\mathrm{GB}\), respectively. Overall, ORCA incurs negligible additional training overhead in this configuration.

\label{sec:implementation-details}

\section{Experiment Validations}
\label{sec:main-results}
In this section, we verify the efficacy of ORCA on three specific issues: (1) Llama pretraining with grid learning rate search and larger scale pretraining. (2) Pretraining task on Qwen3 architecture with the additional QK-norm that may violate the matrix factoring reparameterization. This section focuses on the efficacy of ORCA, and for more robustness test of parameters and controlled interventions on gain of ORCA, readers can refer to Section~\ref{sec:interventions}.

\subsection{Experimental Setup and Tuning Protocol}
\label{sec:setup}

\textbf{Setups.}  We pre-train Llama-series models \cite{touvron2023llama} ranging from 130M to 1.3B parameters on C4 , using 25B to 100B tokens. We also pre-train Qwen3-series models \cite{yang2025qwen3} ranging from 160M to 0.6B parameters on The Pile \citep{gao2020pile}, using  25B to 100B tokens. Finally, we train fine-grained mixture-of-experts models ranging from 0.8B to 8B parameters on OLMo-mix-1124 \citep{olmo20242}, the data mixture including DCLM-Baseline \citep{li2024datacomp} using 25B to 100B tokens.

We first run a learning rate grid-search for Adam, Muon, Pion and ORCA. All trials share the same data, batch size, sequence length, and 50k-step training budget. We use $\mathbf{O}_t = 0.2/\mathrm{Rms}(\mathbf{O}_t)\cdot \mathbf{O}_t$ to rescale the Muon and ORCA update where $\mathbf{O}_t$ is the Newton-Schulz orthogonalization output to align the update with the Adam's \citep{liu2025muon}.
We perform a two-stage coarse-to-fine learning rate grid search. The coarse sweep evaluates $\{5\times10^{-4},10^{-3},2\times10^{-3},5\times10^{-3}\}$. Once the coarse sweep locates a candidate learning rate, we evaluate the neighbors on its both sides to refine the best lr within a smaller local range. The refinement grid consists of $\{6,8\}\times10^{-4}$, $\{1.3,1.6\}\times10^{-3}$,
and $\{3,4\}\times10^{-3}$, which reduces the local spacing between learning rates candidates to about $1.2\times$--$1.5\times$.  All four optimizers use this protocol and the same tuning budget.

The advantage of ORCA holds over a wide range of learning rates: at every learning rate evaluated for both methods, ORCA reaches a lower final validation loss than Muon, and at all evaluated configurations between $10^{-3}$ and $5\times10^{-3}$ it is even better than fully tuned Muon. The independent searches select $1.3\times10^{-3}$, $3\times10^{-3}$, $10^{-3}$, and $2\times10^{-3}$ for Adam, Muon, Pion, and ORCA, with final validation losses of $2.8824$, $2.8460$, $2.9790$, and $2.8106$, respectively; ORCA thus lowers the loss by a further $0.0354$ relative to the best Muon and stays within $0.002$ of its own optimum across $1.6\times10^{-3}$ to $3\times10^{-3}$, showing that the improvement does not rely on a narrow learning-rate choice.

\label{sec:pretraining}

\textbf{LLama Pretraining Tasks.} We compare ORCA against Adam and Muon on LLaMA 350M and 1.3B, training each model for 100B tokens. As shown in \Cref{fig:main-pretraining}, Muon outperforms Adam at both scales, while ORCA further lowers the validation loss. Compared with Muon, ORCA reduces the final loss by $0.0194$ at 350M and $0.0244$ at 1.3B, with consistent gains at both scales. Further, a four-seed robustness result for the 130M, 50k-step setting is reported in
Appendix~\ref{app:more-experiment-results}.
We also evaluate the final ORCA and Muon LLaMA 1.3B checkpoints from \Cref{fig:main-pretraining} on five downstream tasks \citep{clark2018think, zellers2019hellaswag, sakaguchi2021winogrande, bisk2020piqa}. As shown in \Cref{tab:downstream}, ORCA outperforms Muon on all five tasks. This suggests that the lower pre-training loss of ORCA also leads to better downstream performance.

We also evaluate how the optimization performance translates into token efficiency by measuring the training required to reach a common target loss.
We continued ORCA from its original 51k-step checkpoint, starting the final decay early, and reached a validation loss of 2.318 at approximately 60k steps. The standard Muon run, using its original 100k-step settings, reached the same loss at 100k steps (\Cref{fig:token-efficiency}). Therefore, ORCA used 40\% fewer tokens to reach this target, a 1.66× token-efficiency gain. To save compute, we reused the checkpoint without retuning ORCA for the shorter run: its original 25k-step shaping stage accounts for 41.7\% of the 60k-step budget rather than the tuned default of 25\%.

\begin{figure}[t]
  \centering
  \begin{subfigure}[t]{0.32\linewidth}
    \centering
    \includegraphics[width=\linewidth]{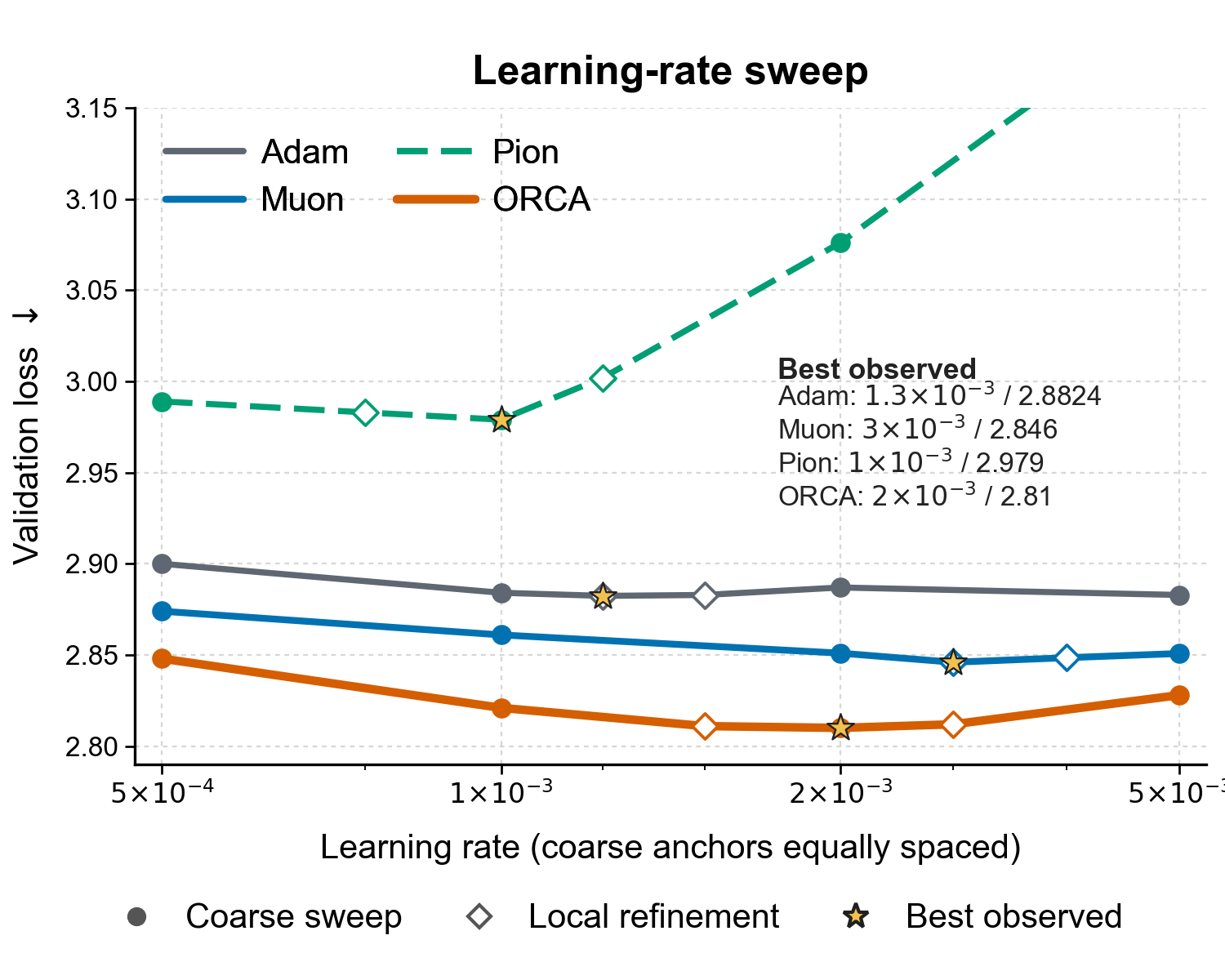}
    \captionsetup{skip=2pt}
    \caption{Learning-rate sweep}
    \label{fig:lr-sweep-panel}
  \end{subfigure}\hfill
  \begin{subfigure}[t]{0.32\linewidth}
    \centering
    \includegraphics[width=\linewidth]{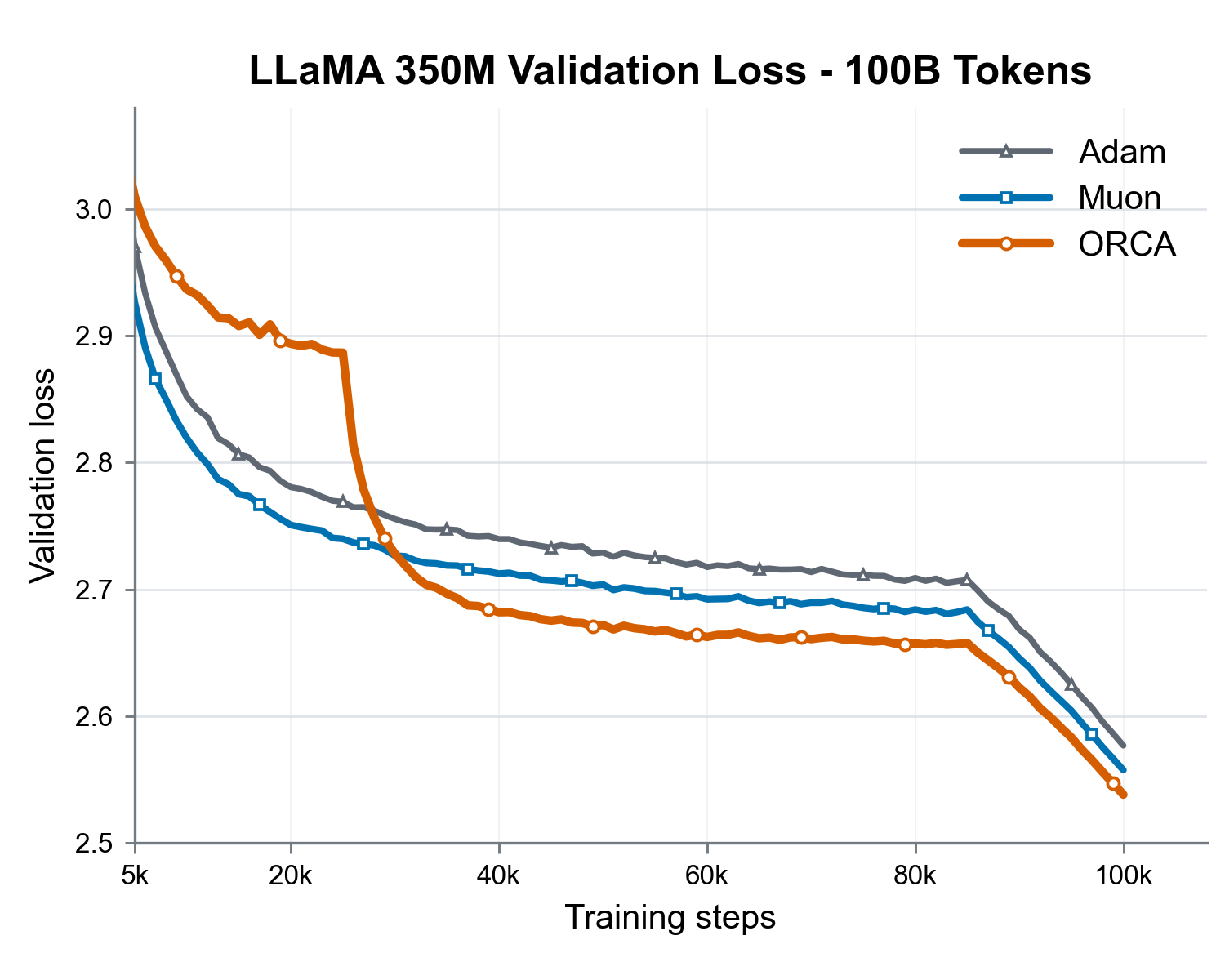}
    \captionsetup{skip=2pt}
    \caption{LLaMA 350M}
    \label{fig:main-350m}
  \end{subfigure}\hfill
  \begin{subfigure}[t]{0.32\linewidth}
    \centering
    \includegraphics[width=\linewidth]{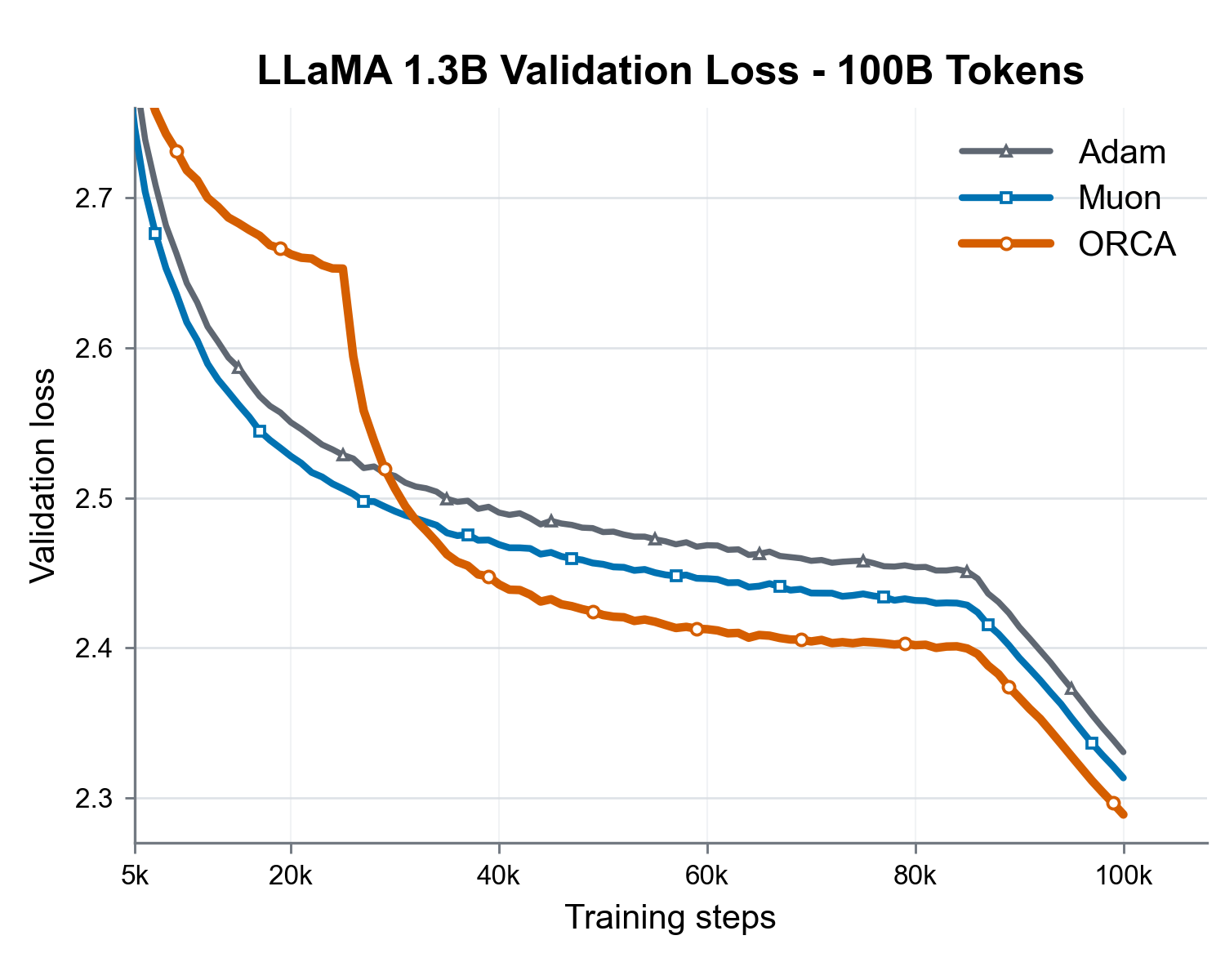}
    \captionsetup{skip=2pt}
    \caption{LLaMA 1.3B}
    \label{fig:main-1b}
  \end{subfigure}
  \captionsetup{skip=2pt}
  \caption{LLaMA pretraining results.}
  \vspace{-8pt}
  \label{fig:lr-sweep}
  \label{fig:main-pretraining}
\end{figure}

\begin{figure}[t]
  \centering
  \begin{subfigure}[t]{0.32\linewidth}
    \centering
    \includegraphics[width=\linewidth]{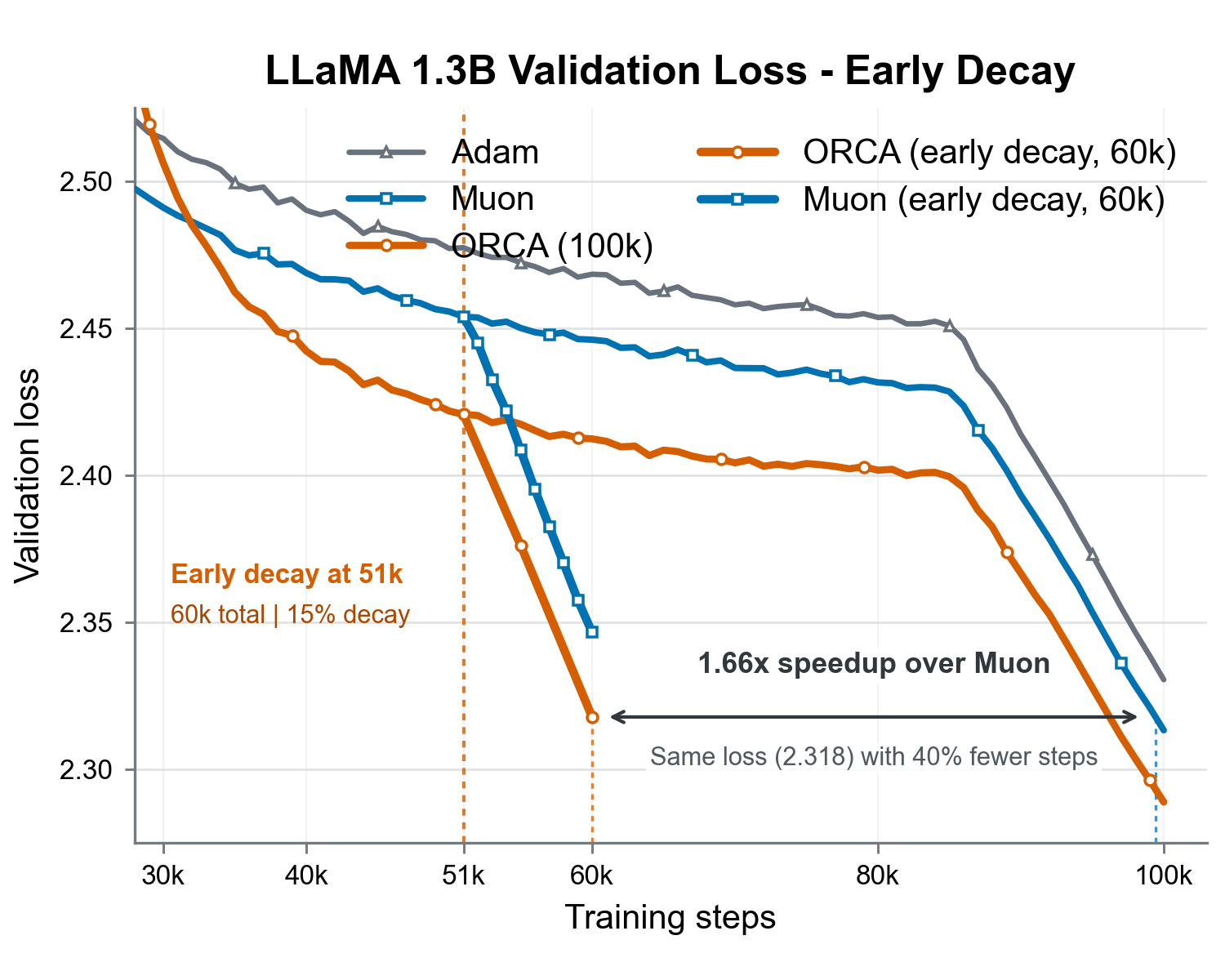}
    \captionsetup{skip=2pt}
    \caption{Matched-target token efficiency}
    \label{fig:token-efficiency}
  \end{subfigure}\hfill
  \begin{subfigure}[t]{0.32\linewidth}
    \centering
    \includegraphics[width=\linewidth]{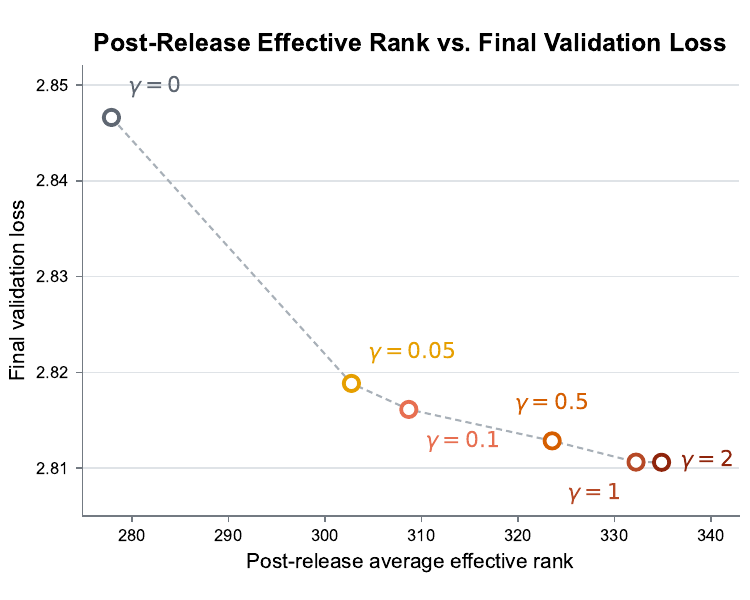}
    \captionsetup{skip=2pt}
    \caption{Shaping strength $\gamma$}
    \label{fig:rank-dose-response}
  \end{subfigure}\hfill
  \begin{subfigure}[t]{0.32\linewidth}
    \centering
    \includegraphics[width=\linewidth]{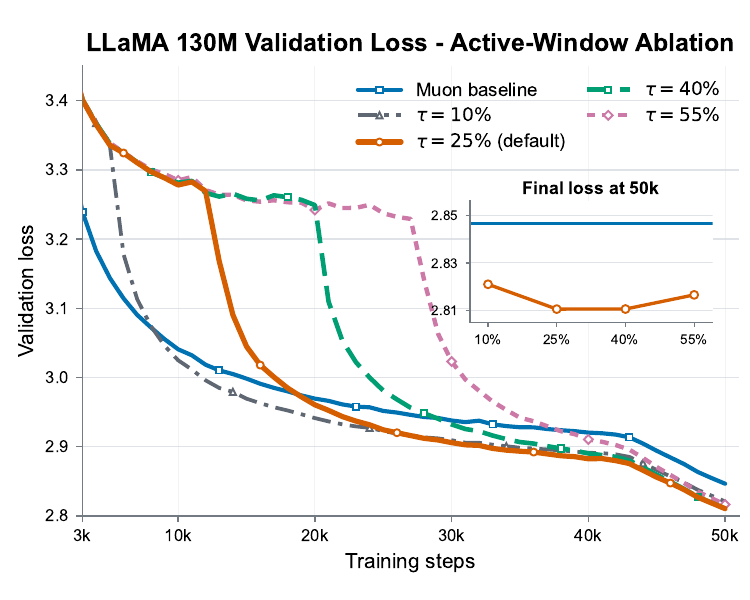}
    \captionsetup{skip=2pt}
    \caption{Active-window fraction $\tau$}
    \label{fig:release-point}
  \end{subfigure}
  \captionsetup{skip=2pt}
  \caption{Token efficiency and shaping ablations in LLaMA Pretraining.}
  \vspace{-15pt}
  \label{fig:shaping-strength-timing}
\end{figure}

\begin{table}[t]
  \centering
  \caption{Downstream evaluation on LLaMA 1.3B.}
  \label{tab:downstream}
  \small
  \setlength{\tabcolsep}{4pt}
  \begin{tabular}{@{}lcccccc@{}}
    \toprule
    Optimizer & ARC-C  & ARC-E & HellaSwag & PIQA & WinoGrande & Avg. \\
    \midrule
    Muon & $0.305\,(0.013)$ & $0.342\,(0.010)$ & $0.549\,(0.005)$
         & $0.705\,(0.011)$ & $0.547\,(0.014)$ & $0.489$ \\
    ORCA & $\mathbf{0.311}\,(0.014)$ & $\mathbf{0.357}\,(0.010)$
         & $\mathbf{0.573}\,(0.005)$ & $\mathbf{0.711}\,(0.011)$
         & $\mathbf{0.553}\,(0.014)$ & $\mathbf{0.501}$ \\
    \bottomrule
  \end{tabular}
  \vspace{-15pt}
  \par\smallskip
\end{table}

\textbf{Qwen3 Pretraining Tasks.} We compare Adam, Muon, and ORCA on Qwen3-style architecture. Within each model scale, all methods share the same common hyperparameters. The Qwen3 0.6B curves cover 50k optimizer steps (about 105B tokens), and the Qwen3 160M curves cover 25k optimizer steps (about 26B tokens).
In both Qwen settings, both sets of loss curves show the same ordering: ORCA may fall temporarily behind Muon early in training, then crosses over and holds the advantage at the
end of the observed interval. The validation-loss improvement of ORCA is therefore not confined to LLaMA, but also reproduces on the Qwen3 architecture.

\textbf{Larger Scale Fine-grained MoE Pretrining.} We observed
router perturbations under a fixed global release point, so we release the
orthogonal constraint layerwise: at $20\%$ of training in the first layer and
at $60\%$ in the last layer, with release times linearly interpolated across
layers. We use top-$4$ routing with $64$ experts in the MoE $0.8$B model
($0.24$B active parameters) and $64$ experts in the MoE $8$B model ($1$B
active parameters) \citep{dai2024deepseekmoe}. Due to computational budget constraints, we focus our comparison on Muon and ORCA. ORCA achieves the lowest final loss at both $25$B and
$50$B tokens for the $0.8$B model, showing that its gain persists as the token
budget increases.
The $8$B experiment extends this result to a larger model trained on $100$B tokens, where ORCA also finishes below Muon (\Cref{fig:moe}). Across these MoE experiments, ORCA reduces final validation loss by $0.020$–$0.025$ relative to Muon.

\textbf{Early orthogonal shaping also benefits other base optimizers.}
We apply the same soft orthogonal shaping step~\eqref{eq:orca-update} to Adam and SOAP on LLaMA 130M \citep{vyas2025soap}. Shaping is active only during the first $25\%$ of training; the base optimizer then continues alone.
For both optimizers, shaping temporarily raises validation loss while active, but the shaped runs overtake their baselines after its removal and finish lower. At 50k steps, it reduces Adam's loss from $2.8824$ to $2.8435$ and SOAP's from $2.8640$ to $2.8271$, gains of $0.0389$ and $0.0369$, respectively. These results show that the benefits of early orthogonal shaping extend beyond Muon to other base optimizers.

\begin{figure}[t]
  \centering
  \begin{subfigure}[t]{0.32\linewidth}
    \centering
    \includegraphics[width=\linewidth]{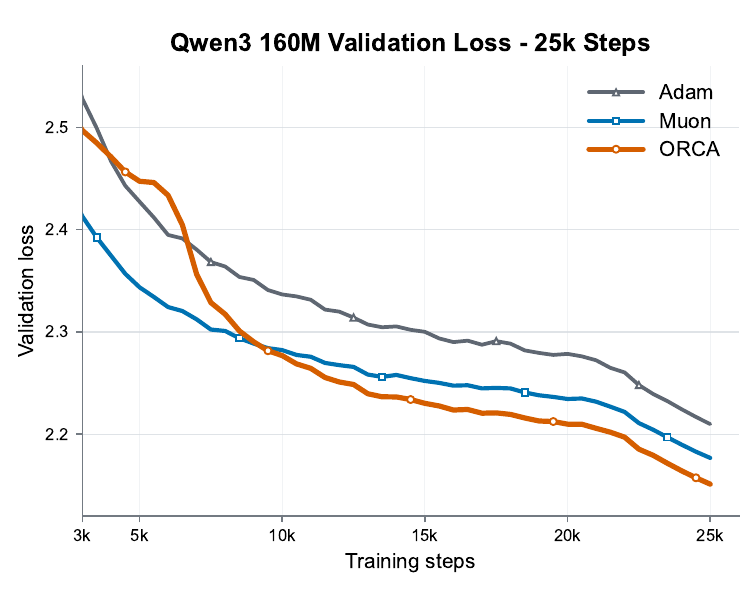}
    \captionsetup{skip=2pt}
    \caption{Qwen3 160M}
    \label{fig:qwen-160m}
  \end{subfigure}\hfill
  \begin{subfigure}[t]{0.32\linewidth}
    \centering
    \includegraphics[width=\linewidth]{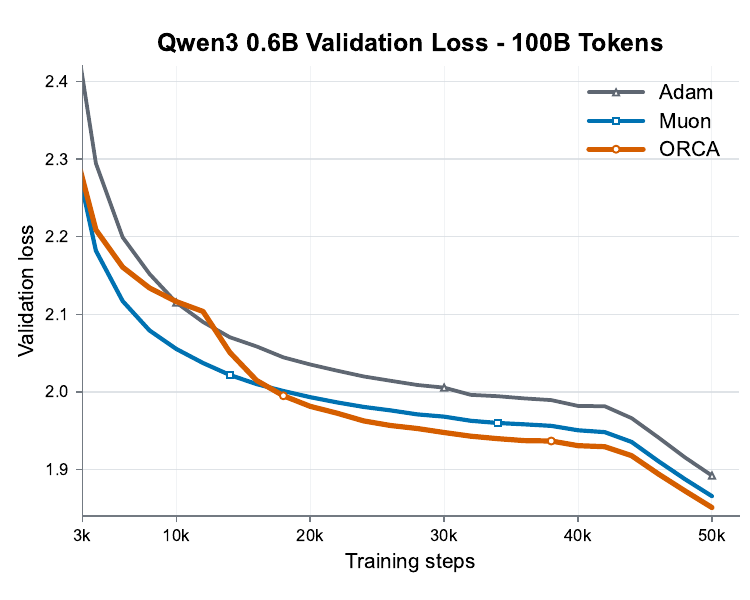}
    \captionsetup{skip=2pt}
    \caption{Qwen3 0.6B}
    \label{fig:qwen-06b}
  \end{subfigure}\hfill
  \begin{subfigure}[t]{0.32\linewidth}
    \centering
    \includegraphics[width=\linewidth]{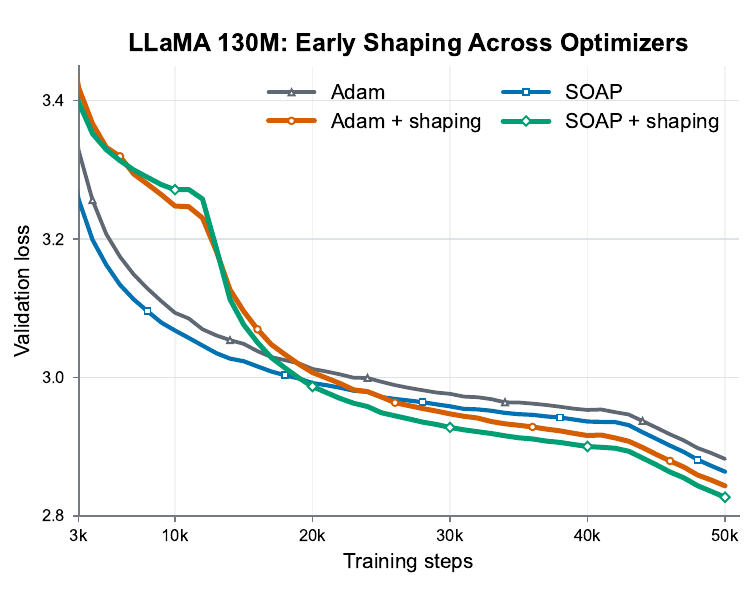}
    \captionsetup{skip=2pt}
    \caption{LLaMA 130M, base optimizers}
    \label{fig:optimizer-generality-panel}
  \end{subfigure}
  \captionsetup{skip=2pt}
  \caption{Validation loss across architectures and base optimizers.}
  \vspace{-15pt}
  \label{fig:qwen}
  \label{fig:optimizer-generality}
\end{figure}

\begin{figure}[t]
  \centering
  \begin{subfigure}[t]{0.32\linewidth}
    \centering
    \includegraphics[width=\linewidth]{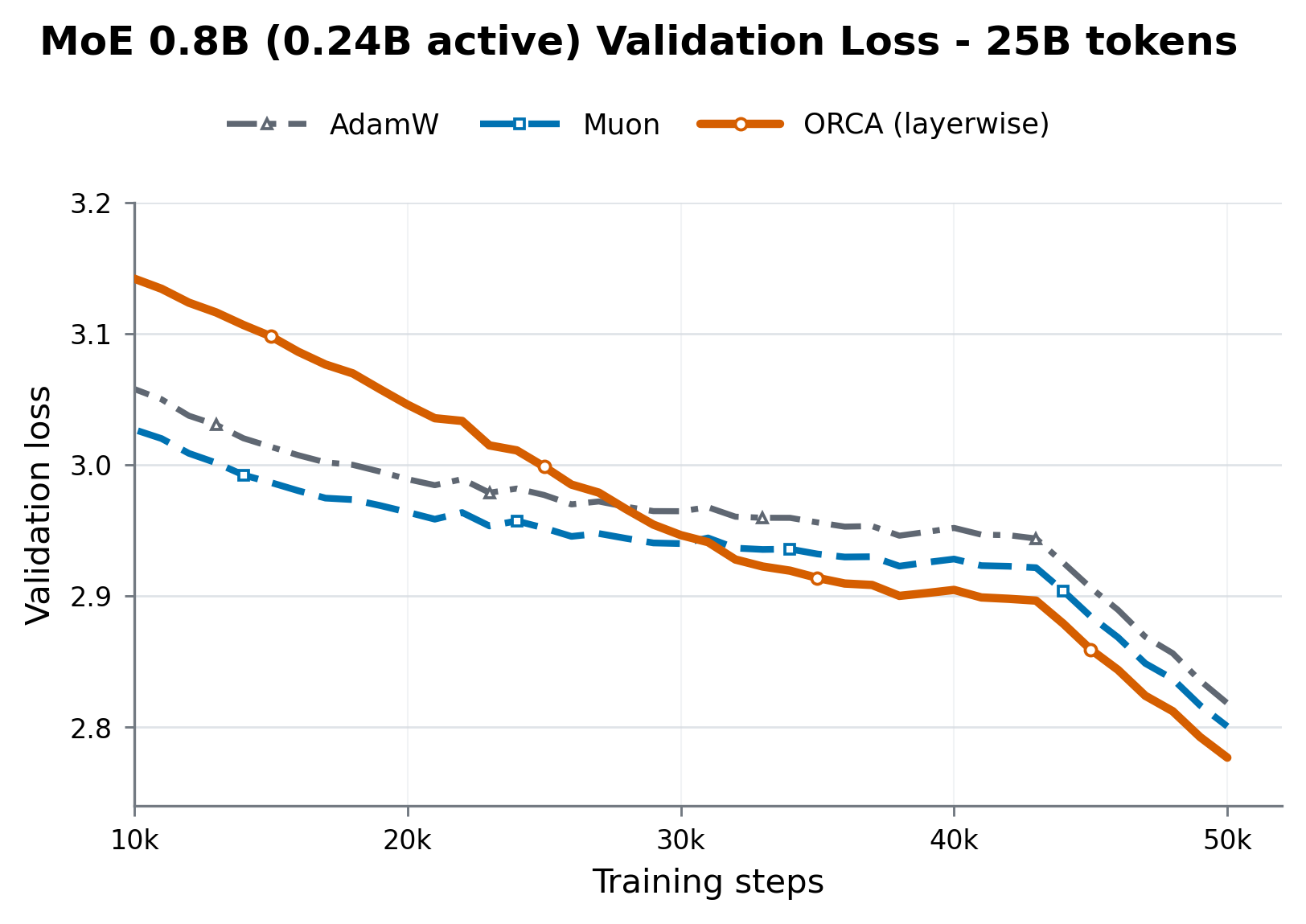}
    \captionsetup{skip=2pt}
    \caption{MoE $0.8$B, $25$B tokens}
    \label{fig:moe-08b-25b}
  \end{subfigure}\hfill
  \begin{subfigure}[t]{0.32\linewidth}
    \centering
    \includegraphics[width=\linewidth]{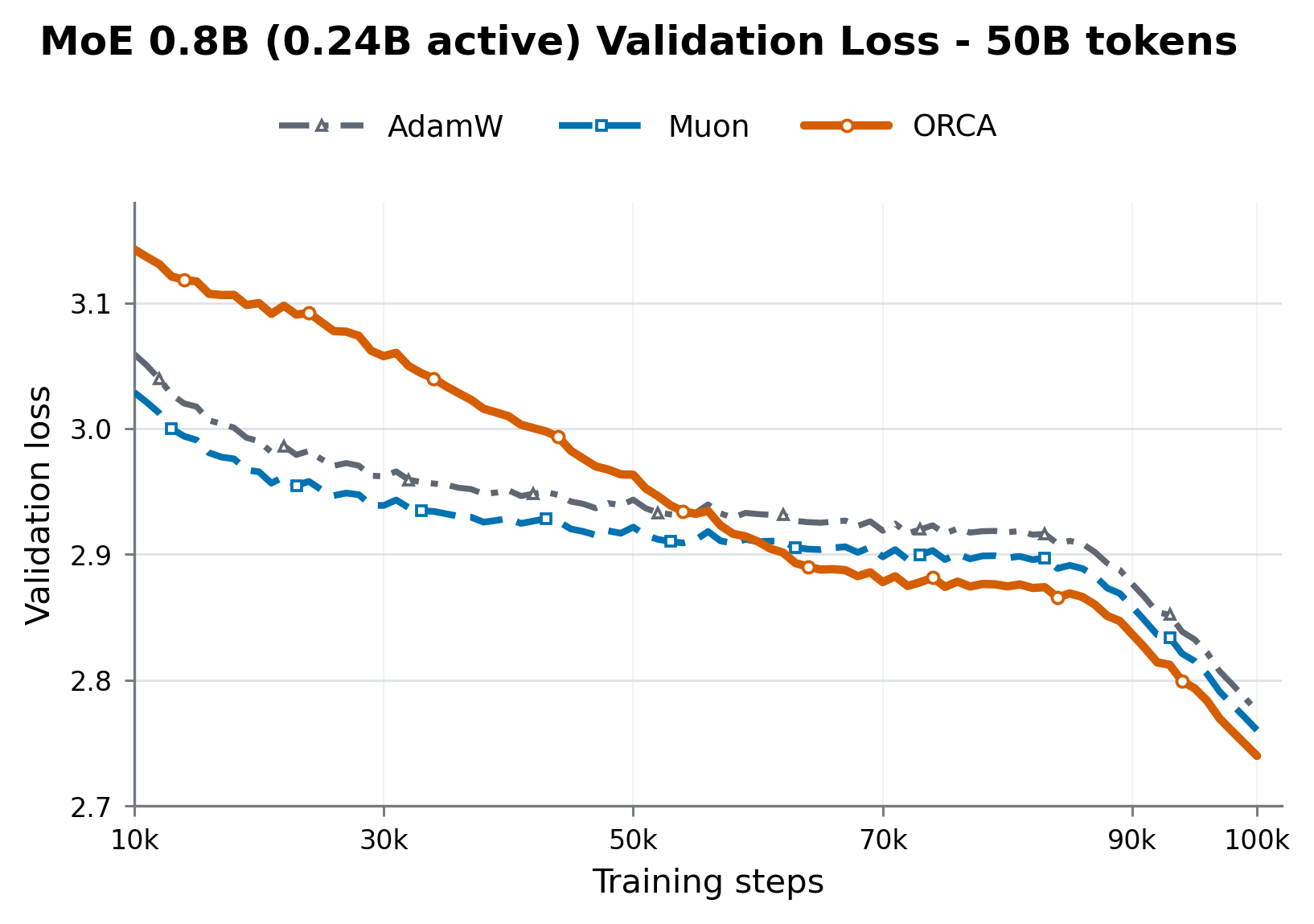}
    \captionsetup{skip=2pt}
    \caption{MoE $0.8$B, $50$B tokens}
    \label{fig:moe-08b-50b}
  \end{subfigure}\hfill
  \begin{subfigure}[t]{0.32\linewidth}
    \centering
    \includegraphics[width=\linewidth]{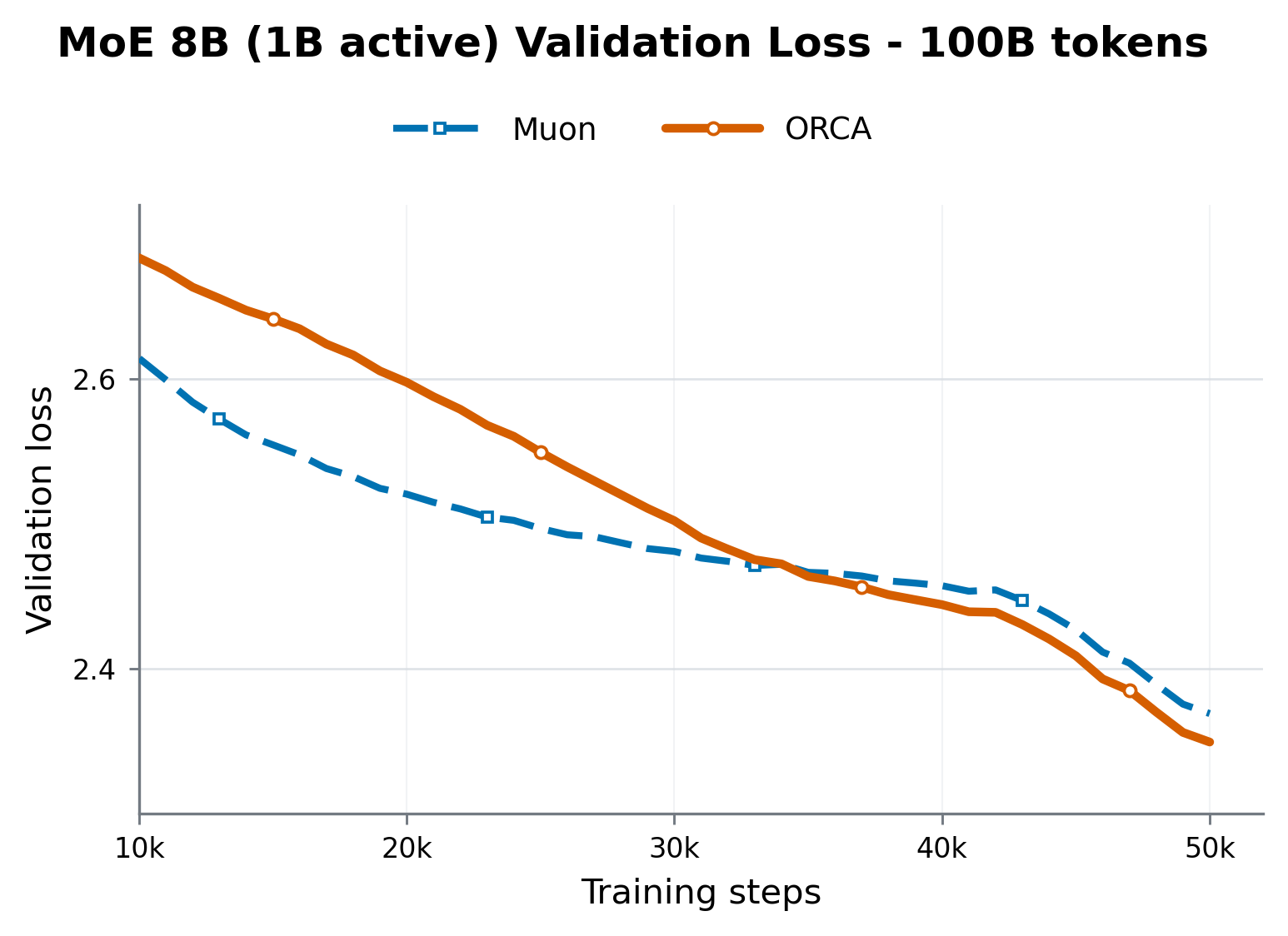}
    \captionsetup{skip=2pt}
    \caption{MoE $8$B $1$B active, $100$B tokens}
    \label{fig:moe-8b}
  \end{subfigure}
  \captionsetup{skip=2pt}
  \caption{Fine-grained MoE validation loss with layer-wise release.}
  \vspace{-15pt}
  \label{fig:moe}
\end{figure}

\section{Where the Gain Comes From: Controlled Interventions}
\label{sec:interventions}

\Cref{fig:intro-rank-loss} offers an initial view of ORCA's mechanism:
early shaping raises the effective rank, part of this increase persists after
shaping is released, and loss then decreases faster. While
\Cref{sec:main-results} establishes ORCA's optimization gains, in this section
we vary shaping strength and timing, as well as the form of the early spectral
intervention, to isolate their contributions to these gains. Additional
controls for optimizer state and learning-rate schedules are reported
separately in Appendix~\ref{app:controls}.

\textbf{Retained Effective Rank Predicts Post-Release Gain}
We first test whether the effective-rank difference retained after release
tracks ORCA's final gain. Holding all other settings fixed, we sweep the shaping
strength over $\gamma\in\{0,0.05,0.1,0.5,1,2\}$ and measure both the average
post-release effective rank and the final validation loss.
The sweep exhibits a clear dose response. As $\gamma$ increases, the retained
effective rank rises from about $278$ to $335$, while the final validation loss
falls from $2.8466$ to $2.8106$ (\Cref{fig:rank-dose-response}). The relation
flattens at larger $\gamma$, indicating diminishing returns. This
supports retained effective rank as an important variable behind ORCA's
post-release gain.

\noindent\textbf{Release-point sensitivity.}
We first vary the duration of the initial shaping stage over
$\tau\in\{0.10,0.25,0.40,0.55\}$ on LLaMA 130M. Every setting improves over
Muon's final loss of $2.8512$: the four release points reach $2.8210$, $2.8106$,
$2.8107$, and $2.8166$, respectively (\Cref{fig:release-point}). The broad
effective range shows that ORCA does not rely on a precise release point. We use
$\tau=0.25$ by default because it is the shortest tested window in the
best-performing range.

\begin{wrapfigure}[13]{r}{0.38\textwidth}
  \vspace{-\intextsep}
  \centering
  \includegraphics[width=\linewidth]{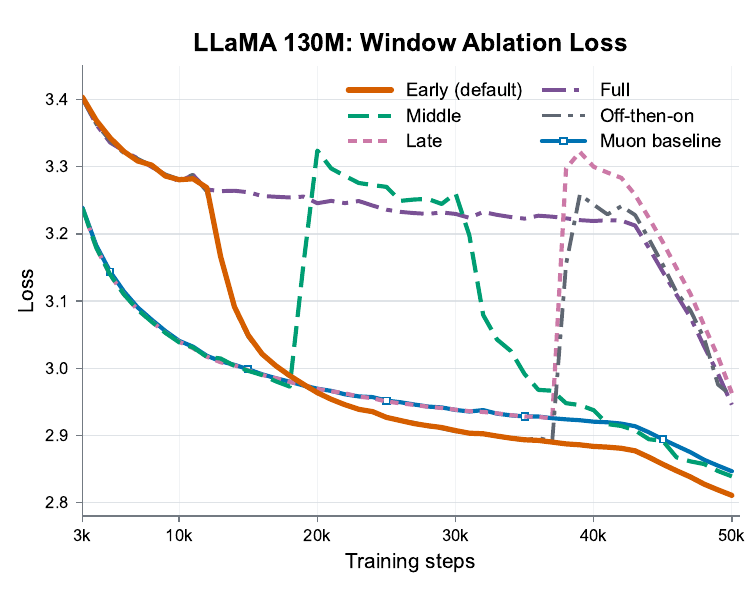}
  \captionsetup{skip=2pt}
  \caption{Loss curves under different shaping windows on LLaMA 130M.}
  \label{fig:window-ablation}
\end{wrapfigure}

\noindent\textbf{Shaping-window ablation.}
We next compare the default early window $[0,0.25T)$ with middle
$[0.375T,0.625T)$, late $[0.75T,T)$, full-training, and off-then-on
$[0,0.25T)\cup[0.75T,T)$ schedules. Together with the effective-rank dynamics
in \Cref{fig:intro-rank-loss}, \Cref{fig:window-ablation} shows two effects qualitatively consistent with the simplified analysis:
\textbf{(i) shaping is costly while active.} Validation loss rises and remains
on a high plateau whenever shaping is imposed, consistent with persistently
high effective rank restricting the task-adaptive spectrum. \textbf{(ii) The
benefit emerges after release and requires time.} Once shaping is removed, the
retained high-rank state supports faster Muon descent without the same endpoint
restriction. Early shaping leaves the longest released phase and achieves the
lowest final loss. Middle shaping also recovers rapidly and finishes slightly
below Muon, but remains worse than early; late, full, and off-then-on shaping
leave insufficient or no final released phase and fail to recover by the
endpoint. These explain ORCA's two-stage design: raise effective rank
early, then reserve most of training for unconstrained optimization.

\vspace{-4pt}

\subsection{Alternative Early Spectral Interventions}
\label{sec:alternative-interventions}
\vspace{-4pt}
ORCA's mechanistic claim extends beyond soft orthogonal regularization: if raising and retaining effective rank early improves optimization after release, other ways of doing so should yield similar gains. We test this with DiagSVD and Pion-start, comparing them with ORCA and standard Muon in 50k-step LLaMA 130M training. Both alternatives replace soft orthogonal shaping during the first $25\%$ of training; afterward, all variants follow the same Muon dynamics.

DiagSVD directly amplifies updates along weaker singular directions. It
reweights the Muon direction $\widetilde{\mathbf{u}}$ in the basis $\mathbf{U}$ of the left
singular vectors of $\mathbf{W}$:
\begin{equation}
  \widehat{\mathbf{u}}=\mathbf{U}\diag(s_i)\mathbf{U}^\top\widetilde{\mathbf{u}}+s_r(\mathbf{I}-\mathbf{U}\mathbf{U}^\top)\widetilde{\mathbf{u}},
  \qquad
  s_i=1+\theta(i/r)^\delta,
  \label{eq:diagsvd}
\end{equation}
where the singular values are in descending order and $\theta=128$,
$\delta=3$. DiagSVD assigns larger update coefficients to directions associated
with smaller singular values, thereby promoting a
higher effective rank of the weight matrix. The exact and lazy variants refresh
the singular-vector basis every step and every five steps, respectively.
Pion-start instead provides a spectrum-preserving intervention. Given the task
gradient $\mathbf{G}$, it constructs the skew-symmetric generators
\begin{equation}
  \mathbf{G}_{\mathrm{in}}=\mathbf{W}^\top \mathbf{G}-(\mathbf{W}^\top \mathbf{G})^\top,
  \qquad
  \mathbf{G}_{\mathrm{out}}=\mathbf{G}\mathbf{W}^\top-(\mathbf{G}\mathbf{W}^\top)^\top,
  \label{eq:pion-generators}
\end{equation}
and updates the weight from both sides as
\begin{equation}
  \mathbf{W}^+=E_2(\mathbf{A}_{\mathrm{out}})\mathbf{W}E_2(\mathbf{A}_{\mathrm{in}}),
  \qquad
  E_2(\mathbf{A})=\mathbf{I}+\mathbf{A}+\tfrac12\mathbf{A}^2,
  \label{eq:pion-start-update}
\end{equation}
\begin{wrapfigure}{r}{0.40\textwidth}
  \vspace{-\intextsep}
  \centering
  \includegraphics[width=\linewidth]{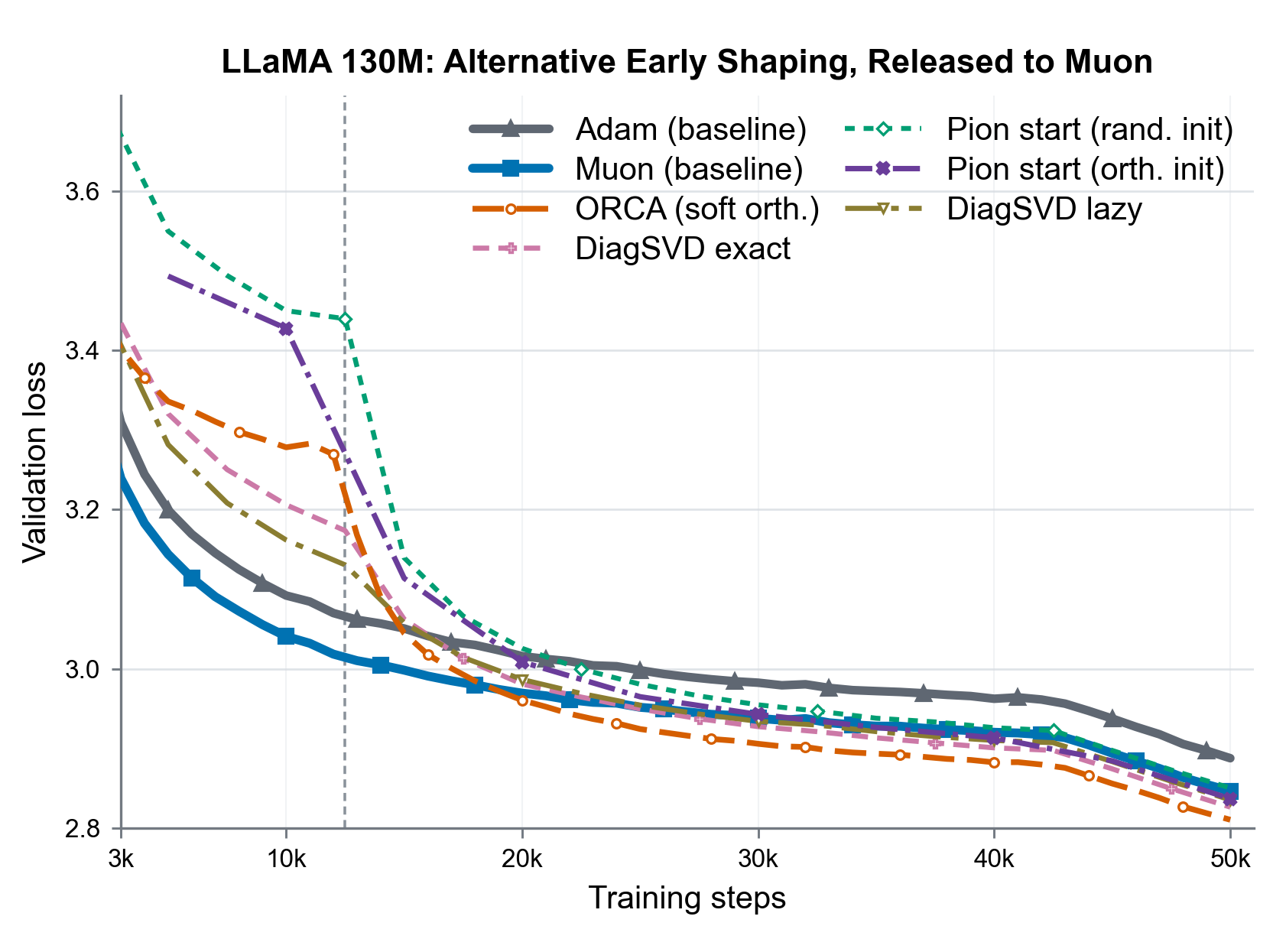}
  \captionsetup{skip=2pt}
  \caption{Alternative early spectral interventions on LLaMA 130M,
  released to Muon after $25\%$ of training.}
  \label{fig:alternative-shaping}
\end{wrapfigure}
where $\mathbf{A}_{\mathrm{in}}$ and $\mathbf{A}_{\mathrm{out}}$ are obtained from momentum and
second-moment estimates of the corresponding generators. The underlying
two-sided Lie action preserves the singular spectrum, and hence the effective
rank, during the intervention period; $E_2$ is its second-order approximation.
For Pion-start based methods, whose underlying dynamics preserve the singular spectrum, an orthogonalized start improves the final loss from $2.8505$ with a random start to $2.8371$. This offers an indication that early orthogonality matters. The two DiagSVD variants recover part of ORCA's gain: against Muon's $2.8512$, DiagSVD exact and lazy reach $2.8265$ and $2.8356$, respectively, while ORCA reaches $2.8106$. By amplifying early updates along directions of smaller singular values, DiagSVD improves the final loss even after release. These all indicate early spectral shaping can benefit later optimization.

\vspace{-4pt}

\section{Further Discussions}

\paragraph{The scale of orthogonal regularization.}
A natural practical question is how to choose the target scale of
orthogonal regularization.
The unit-scale regularizer used above corresponds to \(c=1\) in
\[
    R_{\mathrm{orth}}(\mathbf{W})
    =
    \frac{1}{4}
    \left\|
        \mathbf{W}\mathbf{W}^{\top}
        -
        c^{2}\mathbf{I}
    \right\|_{F}^{2},
    \qquad c > 0,
\]
where \(c\) directly controls the target singular-value scale.
A smaller \(c\) may offer two potential advantages.

First, from a feature-learning perspective, different spectral components
need not develop at the same rate.
Such differences have been analyzed in overparameterized matrix sensing
\citep{li2018algorithmic}.
A smaller \(c\) sets a lower target scale and may help limit the dominance
of rapidly growing directions, potentially facilitating the development
of weaker directions.
This is a possible mechanism, however, and the recovery of weak directions
is not automatically guaranteed.

Second, \(c\) provides an additional degree of freedom for balancing
the effective learning rates of matrix parameters and RMSNorm parameters.
In particular, when
\(\mathbf{W}\in\mathbb{R}^{p\times q}\) with \(p\leq q\),
the target \(\mathbf{W}\mathbf{W}^{\top}\approx c^{2}\mathbf{I}\)
implies \(\operatorname{RMS}(\mathbf{W})\approx c/\sqrt{q}\).
Since Muon's matrix updates are RMS-normalized, their magnitudes are
largely independent of the current weight scale, whereas the relative
update size depends inversely on
\(\operatorname{RMS}(\mathbf{W})\).
Consequently, choosing a smaller \(c\) increases the relative learning
rate of matrix parameters and may help balance their learning dynamics
with those of the RMSNorm scaling parameters, which are typically
optimized separately using AdamW.
This additional flexibility can be useful when the two parameter groups
otherwise evolve at substantially different relative rates.

Empirically, we observe that higher-quality training data tend to favor
a smaller \(c\).
Let \(\mathbf{g}\) denote the learnable per-channel scaling vector
of the corresponding RMSNorm layer.
As a practical heuristic, we find it useful to choose \(c\) slightly below
the maximum absolute value of these scaling parameters,
\(\|\mathbf{g}\|_{\infty}=\max_j |g_j|\).
This is an empirical guideline rather than a universal scaling law.

\paragraph{Dependence on the training budget and learning-rate schedule.}
ORCA's benefit requires sufficient training after release.
In our experiments, runs with approximately \(10{,}000\) optimizer updates
or fewer can underperform Muon; we recommend at least \(30{,}000\) updates
as a practical guideline.
Our default shaping window of \(25\%\) is selected under the WSD schedule
and is not necessarily optimal for other learning-rate schedules.
In particular, the same active-step fraction \(S/T\) can correspond
to different fractions of the total learning-rate budget,
\[
    \rho(S)
    =
    \frac{\sum_{t=0}^{S-1}\eta_t}
         {\sum_{t=0}^{T-1}\eta_t},
\]
which provides a useful reference for comparing and adjusting release
points across schedules.
Nevertheless, Figure~\ref{fig:release-point} shows consistent gains over Muon
across a broad range of release points, indicating robustness to the release fraction.

\paragraph{Matrix granularity and applicability to router weights.}
The granularity at which attention projection matrices are treated
is another important implementation choice.
Among the configurations we searched, treating
\(\mathbf{W}_{Q}\) and \(\mathbf{W}_{K}\) head-wise while retaining
\(\mathbf{W}_{V}\) and \(\mathbf{W}_{O}\) as full matrices performed best.
This configuration outperformed both treating all four projections
head-wise and treating all four as full matrices.
Applying ORCA to router weights also yields benefits in our MoE experiments,
suggesting that its applicability extends beyond attention and
feed-forward projections.

\paragraph{Generality of the two-phase framework.}
We emphasize that the main contribution of this paper is the general
two-phase optimization principle:
introducing a temporary intervention during early training to shape
a favorable optimization trajectory, followed by releasing the intervention
to restore the flexibility needed for task-specific adaptation.
This principle is not tied to any particular optimizer or regularization.

\paragraph{Flexibility in the choice of optimizers.}
On the one hand, the framework can be combined with advanced matrix-based
optimizers, such as SOAP.
On the other hand, it can also be applied to computationally inexpensive
optimizers, such as Adam, without introducing sophisticated matrix-based
update rules.
As demonstrated in our experiments, both SOAP and Adam benefit from
the same early-shaping and subsequent-release strategy.
This flexibility makes the two-phase framework applicable across
different computational budgets and optimization settings.

\paragraph{Flexibility in the choice of regularization.}
Similarly, the orthogonal regularizer used in ORCA is only one possible
realization of early geometric shaping and can be replaced by either
stronger or more lightweight constraints.
On the one hand, stronger geometric constraints, such as optimization
over a possibly scaled Stiefel manifold, provide more explicit control
of the weight geometry and may be particularly beneficial for architectures
prone to severe spectral degeneration.


On the other hand, lightweight alternatives may offer better scalability. For example, row-oblique constraints enforce equal row norms without requiring orthogonality between rows. Other options include analogous column-wise norm constraints and group-wise orthogonality penalties, which encourage orthogonality only within small groups of rows or columns. These alternatives retain partial control over weight geometry without forming the full Gram matrix, although they generally provide weaker protection against spectral degeneration. When chosen to align with the weight-sharding layout, such constraints can reduce both computation and communication costs relative to full Gram-matrix regularization in large-scale model-parallel training.

\section{Conclusion}
\label{sec:conclusion}

ORCA uses temporary early orthogonality regularization to accelerate LLM training, achieving lower validation loss. Its central insight is that optimizers can both choose better update directions and steer early training toward regions where loss is easier to reduce. We hope this perspective inspires future optimizer design.

\newpage

\bibliographystyle{plainnat}
\bibliography{references}

@inproceedings{li2018algorithmic,
  title={Algorithmic regularization in over-parameterized matrix sensing and neural networks with quadratic activations},
  author={Li, Yuanzhi and Ma, Tengyu and Zhang, Hongyang},
  booktitle={Conference on learning theory},
  pages={2--47},
  year={2018},
  organization={PMLR}
}

@article{touvron2023llama,
  title={Llama: Open and efficient foundation language models},
  author={Touvron, Hugo and Lavril, Thibaut and Izacard, Gautier and Martinet, Xavier and Lachaux, Marie-Anne and Lacroix, Timoth{\'e}e and Rozi{\`e}re, Baptiste and Goyal, Naman and Hambro, Eric and Azhar, Faisal and others},
  journal={arXiv preprint arXiv:2302.13971},
  year={2023}
}

@article{yang2025qwen3,
  title={Qwen3 technical report},
  author={Yang, An and Li, Anfeng and Yang, Baosong and Zhang, Beichen and Hui, Binyuan and Zheng, Bo and Yu, Bowen and Gao, Chang and Huang, Chengen and Lv, Chenxu and others},
  journal={arXiv preprint arXiv:2505.09388},
  year={2025}
}

@article{liu2025muon,
  title={Muon is scalable for llm training},
  author={Liu, Jingyuan and Su, Jianlin and Yao, Xingcheng and Jiang, Zhejun and Lai, Guokun and Du, Yulun and Qin, Yidao and Xu, Weixin and Lu, Enzhe and Yan, Junjie and others},
  journal={arXiv preprint arXiv:2502.16982},
  year={2025}
}

@article{clark2018think,
  title={Think you have solved question answering? try arc, the ai2 reasoning challenge},
  author={Clark, Peter and Cowhey, Isaac and Etzioni, Oren and Khot, Tushar and Sabharwal, Ashish and Schoenick, Carissa and Tafjord, Oyvind},
  journal={arXiv preprint arXiv:1803.05457},
  year={2018}
}

@inproceedings{zellers2019hellaswag,
  title={Hellaswag: Can a machine really finish your sentence?},
  author={Zellers, Rowan and Holtzman, Ari and Bisk, Yonatan and Farhadi, Ali and Choi, Yejin},
  booktitle={Proceedings of the 57th annual meeting of the association for computational linguistics},
  pages={4791--4800},
  year={2019}
}

@article{dong2026muon,
  title={Muon{$^{p}$}: Muon with Fractional Spectral Powers},
  author={Dong, Yihe and Sawin, Will},
  journal={arXiv preprint arXiv:2606.13867},
  year={2026}
}

@article{qiu2026reparameterized,
  title={Reparameterized llm training via orthogonal equivalence transformation},
  author={Qiu, Zeju and Buchholz, Simon and Xiao, Tim and Dax, Maximilian and Sch{\"o}lkopf, Bernhard and Liu, Weiyang},
  journal={Advances in Neural Information Processing Systems},
  volume={38},
  pages={140775--140821},
  year={2026}
}

@article{kong2022momentum,
  title={Momentum stiefel optimizer, with applications to suitably-orthogonal attention, and optimal transport},
  author={Kong, Lingkai and Wang, Yuqing and Tao, Molei},
  journal={arXiv preprint arXiv:2205.14173},
  year={2022}
}

@article{gunasekar2017implicit,
  title={Implicit regularization in matrix factorization},
  author={Gunasekar, Suriya and Woodworth, Blake E and Bhojanapalli, Srinadh and Neyshabur, Behnam and Srebro, Nati},
  journal={Advances in neural information processing systems},
  volume={30},
  year={2017}
}

@article{qiu2026poet,
  title={POET-X: Memory-efficient LLM training by scaling orthogonal transformation},
  author={Qiu, Zeju and Liu, Lixin and Weller, Adrian and Shi, Han and Liu, Weiyang},
  journal={arXiv preprint arXiv:2603.05500},
  year={2026}
}

@inproceedings{dai2024deepseekmoe,
  title={Deepseekmoe: Towards ultimate expert specialization in mixture-of-experts language models},
  author={Dai, Damai and Deng, Chengqi and Zhao, Chenggang and Xu, RX and Gao, Huazuo and Chen, Deli and Li, Jiashi and Zeng, Wangding and Yu, Xingkai and Wu, Yu and others},
  booktitle={Proceedings of the 62nd annual meeting of the association for computational linguistics (volume 1: Long papers)},
  pages={1280--1297},
  year={2024}
}

@article{loshchilov2016sgdr,
  title={Sgdr: Stochastic gradient descent with warm restarts},
  author={Loshchilov, Ilya and Hutter, Frank},
  journal={arXiv preprint arXiv:1608.03983},
  year={2016}
}

@article{wang2026pc,
  title={PC Layer: Polynomial Weight Preconditioning for Improving LLM Pre-Training},
  author={Wang, Senmiao and Fang, Tiantian and Zhang, Haoran and Zhang, Yushun and Zhao, Kunxiang and Schwing, Alex and Sun, Ruoyu},
  journal={arXiv preprint arXiv:2606.06470},
  year={2026}
}

@article{peyre2026muon,
  title={Muon Dynamics as a Spectral Wasserstein Flow},
  author={Peyr{\'e}, Gabriel},
  journal={arXiv preprint arXiv:2604.04891},
  year={2026}
}

@book{ambrosio2005gradient,
  title={Gradient flows: in metric spaces and in the space of probability measures},
  author={Ambrosio, Luigi and Gigli, Nicola and Savar{\'e}, Giuseppe},
  year={2005},
  publisher={Springer}
}

@article{sakaguchi2021winogrande,
  title={Winogrande: An adversarial winograd schema challenge at scale},
  author={Sakaguchi, Keisuke and Bras, Ronan Le and Bhagavatula, Chandra and Choi, Yejin},
  journal={Communications of the ACM},
  volume={64},
  number={9},
  pages={99--106},
  year={2021},
  publisher={ACM New York, NY, USA}
}

@article{gao2020pile,
  title={The pile: An 800gb dataset of diverse text for language modeling},
  author={Gao, Leo and Biderman, Stella and Black, Sid and Golding, Laurence and Hoppe, Travis and Foster, Charles and Phang, Jason and He, Horace and Thite, Anish and Nabeshima, Noa and others},
  journal={arXiv preprint arXiv:2101.00027},
  year={2020}
}

@article{hu2024minicpm,
  title={Minicpm: Unveiling the potential of small language models with scalable training strategies},
  author={Hu, Shengding and Tu, Yuge and Han, Xu and He, Chaoqun and Cui, Ganqu and Long, Xiang and Zheng, Zhi and Fang, Yewei and Huang, Yuxiang and Zhao, Weilin and others},
  journal={arXiv preprint arXiv:2404.06395},
  year={2024}
}

@article{loshchilov2017decoupled,
  title={Decoupled weight decay regularization},
  author={Loshchilov, Ilya and Hutter, Frank},
  journal={arXiv preprint arXiv:1711.05101},
  year={2017}
}

@article{li2024datacomp,
  title={Datacomp-lm: In search of the next generation of training sets for language models},
  author={Li, Jeffrey and Fang, Alex and Smyrnis, Georgios and Ivgi, Maor and Jordan, Matt and Gadre, Samir and Bansal, Hritik and Guha, Etash and Keh, Sedrick and Arora, Kushal and others},
  journal={Advances in Neural Information Processing Systems},
  volume={37},
  pages={14200--14282},
  year={2024}
}

@article{zhu2026iso,
  title={ISO: An RLVR-Native Optimization Stack},
  author={Zhu, Hanqing and Cong, Wenyan and Sha, Zhizhou and Mukherjee, Sagnik and Song, Xinyuan and Gonz{\'a}lez-Mart{\'\i}nez, David and Wu, Xiaoxia and Tian, Yuandong and Liu, Shiwei and Pan, David Z and others},
  journal={arXiv preprint arXiv:2607.19331},
  year={2026}
}

@inproceedings{ainslie2023gqa,
  title={Gqa: Training generalized multi-query transformer models from multi-head checkpoints},
  author={Ainslie, Joshua and Lee-Thorp, James and De Jong, Michiel and Zemlyanskiy, Yury and Lebr{\'o}n, Federico and Sanghai, Sumit},
  booktitle={Proceedings of the 2023 conference on empirical methods in natural language processing},
  pages={4895--4901},
  year={2023}
}

@article{shi2026pion,
  title={Pion: A spectrum-preserving optimizer via orthogonal equivalence transformation},
  author={Shi, Kexuan and Li, Hanxuan and Qiu, Zeju and Wen, Yandong and Buchholz, Simon and Liu, Weiyang},
  journal={arXiv preprint arXiv:2605.12492},
  year={2026}
}

@article{olmo20242,
  title={2 OLMo 2 Furious},
  author={OLMo, Team and Walsh, Pete and Soldaini, Luca and Groeneveld, Dirk and Lo, Kyle and Arora, Shane and Bhagia, Akshita and Gu, Yuling and Huang, Shengyi and Jordan, Matt and others},
  journal={arXiv preprint arXiv:2501.00656},
  year={2024}
}

@inproceedings{bisk2020piqa,
  title={Piqa: Reasoning about physical commonsense in natural language},
  author={Bisk, Yonatan and Zellers, Rowan and Gao, Jianfeng and Choi, Yejin and others},
  booktitle={Proceedings of the AAAI conference on artificial intelligence},
  volume={34},
  number={05},
  pages={7432--7439},
  year={2020}
}

@inproceedings{vyas2025soap,
  title={SOAP: Improving and stabilizing shampoo using adam for language modeling},
  author={Vyas, Nikhil and Morwani, Depen and Zhao, Rosie and Shapira, Itai and Brandfonbrener, David and Janson, Lucas and Kakade, Sham},
  booktitle={International Conference on Learning Representations},
  volume={2025},
  pages={93423--93444},
  year={2025}
}

@article{jordan2024muon,
  title={Muon: An optimizer for hidden layers in neural networks, 2024},
  author={Jordan, Keller and Jin, Yuchen and Boza, Vlado and You, Jiacheng and Cesista, Franz and Newhouse, Laker and Bernstein, Jeremy},
  journal={URL https://kellerjordan. github. io/posts/muon},
  volume={6},
  number={3},
  pages={4},
  year={2024}
}

\newpage

\appendix
\crefalias{section}{appendix}
\crefalias{subsection}{appendix}

\section{Orthogonal and spectral training.}
\citet{kong2022momentum} developed a momentum optimizer that maintains orthogonality constraints. For LLMs, POET and POET-X preserve weight spectra through orthogonal reparameterization, while Pion preserves them through optimizer updates \citep{qiu2026reparameterized,qiu2026poet,shi2026pion}. ISO applies fixed-spectrum optimization to post-training \citep{zhu2026iso}. These methods report impressive training stability and downstream or post-training results; Pion and ISO outperform Muon in several downstream or post-training evaluations. However, persistent spectral constraints have not yet been shown to consistently achieve lower LLM pretraining loss than Muon. Similar attention to weight spectra appears at the model level: the PC layer improves convergence by carefully reparameterizing weights to reshape their spectra \citep{wang2026pc}. Optimizer-side shaping may be more broadly applicable and less costly because it does not modify the model's forward or backward passes. One key difference in our method is that we remove the orthogonal constraint entirely after early training. We believe this freedom later in training is an important reason for our larger improvement over Muon.

\section{Experiment Setup}
\label{sec:settings}

This appendix summarizes the standard pretraining configurations used in the
main experiments. We omit the controlled ablations and optimizer-specific
interventions discussed in the main text. Unless stated otherwise, all runs
use a decoder-only Transformer with RoPE, SwiGLU feed-forward blocks, RMSNorm,
and bias-free linear layers. The matrix parameters use Muon or ORCA, while
embeddings, the language-model head, and one-dimensional parameters follow the
AdamW fallback described in the method section.

\subsection{Model and Dataset}

We train LLaMA-style models on C4, Qwen3-style models on The Pile, and the
fine-grained MoE models on the olmo-mix-1124 mixture \citep{gao2020pile, olmo20242}. The corresponding
tokenizers are the T5 SentencePiece Unigram tokenizer (vocabulary size
$32{,}100$), the Qwen3 tokenizer (vocabulary size $151{,}936$), and the OLMo
tokenizer (vocabulary size $100{,}278$).

For the LLaMA models, attention is multi-head with equal numbers of query and
key/value heads, the RoPE base is $10{,}000$, and the RMSNorm epsilon is
$10^{-6}$. The model configurations used in the paper are:

\begin{table}[t]
  \centering
  \small
  \caption{LLaMA configurations used in the standard experiments.}
  \label{tab:settings-llama}
  \begin{tabular}{lrrrrr}
    \toprule
    Model & Layers & Hidden & FFN & Q/KV heads & Head dim \\
    \midrule
    130M & 12 & 768  & 2048 & 12/12 & 64 \\
    350M & 24 & 1024 & 2736 & 16/16 & 64 \\
    1.3B & 24 & 2048 & 5461 & 32/32 & 64 \\
    \bottomrule
  \end{tabular}
\end{table}

Qwen3 uses grouped-query attention with head dimension 128 \citep{ainslie2023gqa}, a RoPE base of
$10^6$, RMSNorm epsilon $10^{-6}$, and the default QK normalization. The
MoE models also use grouped-query attention. Their router uses sigmoid scores,
top-$4$ routing, an auxiliary loss coefficient of $10^{-3}$, expert bias, a
top-$k$ scaling factor of 2, and one shared expert per layer. The first layer
is dense and the remaining layers are MoE layers.

\begin{table}[t]
  \centering
  \small
  \caption{Qwen3 and fine-grained MoE configurations.}
  \label{tab:settings-qwen-moe}
  \begin{tabular}{lrrrrrr}
    \toprule
    Model & Layers & Hidden & Dense FFN & Q/KV heads & Experts & Top-$k$ \\
    \midrule
    Qwen3 160M & 20 & 512  & 1536 & 8/4  & -- & -- \\
    Qwen3 0.6B & 28 & 1024 & 3072 & 16/8 & -- & -- \\
    MoE 0.8B & 16 & 768  & 1536 & 6/3  & 64 & 4 \\
    MoE 8B & 27 & 2048 & 4096 & 16/8 & 64 & 4 \\
    \bottomrule
  \end{tabular}
\end{table}

For the MoE models, the expert and shared-expert FFN widths are 288 for the
0.8B model and 768 for the 8B model. The MoE models use a RoPE base of
$10^6$, maximum context length 40960, and RMSNorm epsilon $10^{-6}$; their
input embeddings and output heads are untied. Qwen3 ties its input and output
embeddings. The active parameter counts of the two MoE models are 0.24B and
1B, respectively.

\subsection{Training Setup}

All standard runs use a 1,000-step linear warmup followed by a WSD schedule.
The final 15\% of the training steps is a linear decay from the peak learning
rate to a minimum learning rate equal to one tenth of the peak. Thus, the
middle portion of training uses the peak learning rate. The cosine schedule
appears only in the omitted schedule-control ablation.

Table~\ref{tab:settings-budget} reports the exact number of processed tokens
implied by the global batch size, sequence length, and plotted training steps.
The figure titles use rounded budgets (25B, 50B, or 100B tokens), whereas the
values in the table are not rounded.

\begin{table}[t]
  \centering
  \scriptsize
  \resizebox{\linewidth}{!}{%
  \begin{tabular}{l l r r r r r r}
    \toprule
    Model & Data & Seq. len. & Global batch & Steps & Tokens & Peak LR & WD \\
    \midrule
    LLaMA 130M & C4 & 1024 & 512  & 50,000  & 26.2144B  & searched & 0.01 \\
    LLaMA 350M & C4 & 1024 & 1024 & 100,000 & 104.8576B & $1\!\times\!10^{-3}$ & 0.1 \\
    LLaMA 1.3B & C4 & 1024 & 1024 & 100,000 & 104.8576B & $6\!\times\!10^{-4}$ & 0.1 \\
    Qwen3 160M & The Pile & 4096 & 256 & 25,000 & 26.2144B & $1\!\times\!10^{-3}$ & 0.01 \\
    Qwen3 0.6B & The Pile & 4096 & 512 & 50,000 & 104.8576B & $6\!\times\!10^{-4}$ & 0.1 \\
    MoE 0.8B & olmo-mix-1124 & 1024 & 512 & 50,000 & 26.2144B & $1\!\times\!10^{-3}$ & 0.1 \\
    MoE 0.8B & olmo-mix-1124 & 1024 & 512 & 100,000 & 52.4288B & $1\!\times\!10^{-3}$ & 0.1 \\
    MoE 8B & olmo-mix-1124 & 2048 & 1024 & 50,000 & 104.8576B & $4\!\times\!10^{-4}$ & 0.1 \\
    \bottomrule
  \end{tabular}}
  \caption{Training budgets for the standard experiments. Tokens are computed
  as sequence length times global batch size times optimizer steps.}
  \label{tab:settings-budget}
\end{table}

For standard ORCA runs, the shaping strength is $\gamma=1$ and the active
window is the first $25\%$ of training. In the learning-rate sweep, the Muon
and ORCA matrix updates are additionally rescaled by
$0.2\,\mathbf{O}_t/\operatorname{Rms}(\mathbf{O}_t)$, following the matched
update normalization used in the main text. For the LLaMA 130M learning-rate
sweep, the selected peak learning rates are
$1.3\times10^{-3}$ for Adam, $3\times10^{-3}$ for Muon,
$1\times10^{-3}$ for Pion, and $2\times10^{-3}$ for ORCA. The remaining
models use the peak rates listed in Table~\ref{tab:settings-budget}.

\subsection{Implementation Details}

All standard Muon and ORCA runs use eight Newton--Schulz iterations with the
Polar Express coefficients and Nesterov momentum of $0.95$. The ORCA shaping
window is released layerwise for the MoE experiments: the first layer is
released at 20\% of training, the last layer at 60\%, and intermediate layers
use linear interpolation. The MoE 8B experiment compares Muon and ORCA under
the available budget; the 0.8B experiments additionally include AdamW as shown
in the main figure.

\section{More Experiment Results}
\label{app:more-experiment-results}

\subsection{Robustness across random seeds}
\label{app:multiseed-130m}

We repeated the LLaMA 130M experiment
for four random seeds under the same 50k-step training protocol. Table~\ref{tab:multiseed-130m}
reports the final validation loss for each seed together with the mean and the
sample variance across seeds. ORCA achieves the lowest mean loss and the smallest
across-seed variance. Relative to Muon, ORCA lowers the mean final loss by
$0.0387$ and reduces the sample variance by approximately an order of magnitude.
The ordering ORCA $<$ Muon $<$ Adam is preserved for every seed, indicating that
the observed improvement is not tied to a particular initialization.

\begin{table}[t]
  \centering
  \caption{Final validation loss across four random seeds for LLaMA 130M at
  50k steps. The final two columns report the mean and sample variance across
  seeds; summary statistics are computed from the displayed values.}
  \label{tab:multiseed-130m}
  \scriptsize
  \setlength{\tabcolsep}{4pt}
  \begin{tabular}{@{}lrrrrrr@{}}
    \toprule
    Method & Seed 0 & Seed 1 & Seed 2 & Seed 3 & Mean & Variance \\
    \midrule
    Adam & 2.882 & 2.880 & 2.877 & 2.884 & 2.8808 & $8.92\times10^{-6}$ \\
    Muon & 2.846 & 2.851 & 2.852 & 2.853 & 2.8505 & $9.67\times10^{-6}$ \\
    ORCA & 2.811 & 2.813 & 2.812 & 2.811 & $\mathbf{2.8118}$ & $\mathbf{9.17\times10^{-7}}$ \\
    \bottomrule
  \end{tabular}
  \par\smallskip
  {\footnotesize Variance is the sample variance over the four seeds.}
\end{table}

\subsection{Optimizer-State and Learning-Rate Controls}
\label{app:controls}

We test whether ORCA's gain depends on optimizer state carried across release
or on the WSD learning-rate schedule \citep{hu2024minicpm, loshchilov2016sgdr}. \Cref{tab:controls} reports the final
validation loss at 50k steps.

\begin{table}[t]
  \centering
  \caption{Optimizer-state and learning-rate-schedule controls.}
  \label{tab:controls}
  \resizebox{\textwidth}{!}{%
  \begin{tabular}{lllccc}
    \toprule
    Question & Reference setting & Variant setting
      & Reference loss $\downarrow$ & Variant loss $\downarrow$
      & Variant $-$ Reference \\
    \midrule
    Optimizer state & ORCA, state retained & ORCA, state reset
      & 2.8107 & 2.8106 & $-0.0001$ \\
    Schedule: WSD & Muon & ORCA & 2.8512 & 2.8106 & $-0.0406$ \\
    Schedule: cosine & Muon & ORCA & 2.8595 & 2.8139 & $-0.0456$ \\
    \bottomrule
  \end{tabular}}
\end{table}

\textbf{The gain of ORCA does not come from carrying over the optimizer state.}
At the end of the active window we keep the model parameters unchanged and
reinitialize only the Muon optimizer state, then compare against a matched run
that retains that state. The final validation losses at 50k steps are $2.8106$
with the state reset and $2.8107$ with it retained, a difference of
$1\times10^{-4}$. Resetting the optimizer state does not remove the advantage of
ORCA over Muon, so that advantage cannot be attributed to the optimizer buffers
inherited at release.

\textbf{The gain of ORCA does not depend on a particular learning-rate
schedule.} Under both the WSD and the cosine schedule, ORCA is better than the
corresponding Muon run, with a final validation loss lower by $0.0406$ and
$0.0456$ respectively. The advantage of ORCA therefore does not depend on the
specific decay shape of WSD.

\section{Proof of the Path-Dependent Loss Lower Bound}
\label{app:path}

\subsection{The Exact-Polar Muon Continuous-Time Flow}
\label{app:polar-flow}

For any matrix $\mathbf{A}$, let $\mathbf{Q}(\mathbf{A})$ denote a polar direction satisfying
\begin{equation}
  \ip{\mathbf{A}}{\mathbf{Q}(\mathbf{A})}=\Nnorm{\mathbf{A}},
  \qquad
  \Snorm{\mathbf{Q}(\mathbf{A})}\le1.
  \label{eq:polar-def}
\end{equation}
If $\mathbf{A}=\mathbf{U}\mathbf{\Sigma} \mathbf{V}^\top$ is a thin singular value decomposition, one may take
$\mathbf{Q}(\mathbf{A})=\mathbf{U}\mathbf{V}^\top$; in the rank-deficient case, an element satisfying
\eqref{eq:polar-def} can be selected from the subdifferential of the nuclear
norm. The argument below uses this defining identity directly for any valid
choice of polar direction.

Let $\mathbf{G}_t=\nabla_{\mathbf{W}}L(\mathbf{X}_t\mathbf{Y}_t)$. The exact-polar Muon continuous-time flow with
per-factor scheduling is defined by requiring, almost everywhere,
\begin{equation}
  \dot{\mathbf{X}}_t=-\eta_t^X\mathbf{Q}(\mathbf{G}_t\mathbf{Y}_t^\top),
  \qquad
  \dot{\mathbf{Y}}_t=-\eta_t^Y\mathbf{Q}(\mathbf{X}_t^\top \mathbf{G}_t).
  \label{eq:polar-flow}
\end{equation}
We assume that $k\ge2$, $(\mathbf{X}_t,\mathbf{Y}_t)$ is absolutely continuous on $[0,T]$,
and that the solution exists with $\eta_0\ge0$ and
$\Snorm{\mathbf{X}_t},\Snorm{\mathbf{Y}_t}>0$. The per-factor step-size
bounds are the same as those in Theorem~\ref{thm:mechanistic}:
\[
  0\le\eta_t^X\le\frac{\eta_0}{\Snorm{\mathbf{Y}_t}},
  \qquad
  0\le\eta_t^Y\le\frac{\eta_0}{\Snorm{\mathbf{X}_t}}.
\]
The common step-size condition
\begin{equation}
  \eta_t^X=\eta_t^Y=\eta_t
  \quad\text{with}\quad
  0\le\eta_t\le\frac{\eta_0}{\max\{\Snorm{\mathbf{X}_t},\Snorm{\mathbf{Y}_t}\}}
  \label{eq:common-stepsize}
\end{equation}
is a special case.

\subsection{Two Auxiliary Lemmas}
\label{app:lemmas}

\begin{lemma}[Nuclear norm of a product]
\label{lem:nuclear-product}
For any conformable matrices $\mathbf{A},\mathbf{B}$,
\begin{equation}
  \Nnorm{\mathbf{A}\mathbf{B}}\le\Fnorm{\mathbf{A}}\Fnorm{\mathbf{B}}.
  \label{eq:nuclear-product}
\end{equation}
\end{lemma}

\begin{proof}
By the duality between the nuclear and spectral norms, for any $\Snorm{\mathbf{M}}\le1$,
\begin{equation*}
  \ip{\mathbf{A}\mathbf{B}}{\mathbf{M}}
  =\ip{\mathbf{A}}{\mathbf{M}\mathbf{B}^\top}
  \le\Fnorm{\mathbf{A}}\Fnorm{\mathbf{M}\mathbf{B}^\top}
  \le\Fnorm{\mathbf{A}}\Snorm{\mathbf{M}}\Fnorm{\mathbf{B}}
  \le\Fnorm{\mathbf{A}}\Fnorm{\mathbf{B}}.
\end{equation*}
Taking the supremum over $\mathbf{M}$ gives \eqref{eq:nuclear-product}.
\end{proof}

\begin{lemma}[Gradient bound for the cross-entropy objective]
\label{lem:ce-gradient}
In the setting of \eqref{eq:objective}, if $\mathbf{G}=\nabla_{\mathbf{W}}L(\mathbf{W})$ then
\begin{equation}
  \Fnorm{\mathbf{G}}\le\sqrt2\,L(\mathbf{W}).
  \label{eq:ce-gradient}
\end{equation}
\end{lemma}

\begin{proof}
Let $\mathbf{p}_i$ be the softmax vector of column $i$, so that the corresponding gradient
is $\mathbf{g}_i=\mathbf{p}_i-\mathbf{e}_{c_i}$. Then
\begin{equation*}
  \begin{aligned}
    \|\mathbf{g}_i\|_2^2
    &=(1-p_{c_i i})^2+\sum_{j\ne c_i}p_{ji}^2\\
    &\le(1-p_{c_i i})^2+\left(\sum_{j\ne c_i}p_{ji}\right)^2
    =2(1-p_{c_i i})^2.
  \end{aligned}
\end{equation*}
Since $p_{c_i i}=e^{-\ell_i}$ and $1-e^{-x}\le x$ for $x\ge0$, we get
$\|\mathbf{g}_i\|_2\le\sqrt2\,\ell_i$. Therefore
\begin{equation*}
  \Fnorm{\mathbf{G}}
  =\left(\sum_i\|\mathbf{g}_i\|_2^2\right)^{1/2}
  \le\sqrt2\left(\sum_i\ell_i^2\right)^{1/2}
  \le\sqrt2\sum_i\ell_i
  =\sqrt2L(\mathbf{W}),
\end{equation*}
which is \eqref{eq:ce-gradient}.
\end{proof}

\subsection{Proof of \texorpdfstring{\Cref{thm:mechanistic}}{Theorem 2.1}}
\label{app:path-proof}

Since $L(\mathbf{X}\mathbf{Y})$ is smooth in $(\mathbf{X},\mathbf{Y})$, the chain rule applies almost everywhere
along an absolutely continuous path:
\begin{equation}
  \dot L_t=\ip{\mathbf{G}_t\mathbf{Y}_t^\top}{\dot{\mathbf{X}}_t}+\ip{\mathbf{X}_t^\top \mathbf{G}_t}{\dot{\mathbf{Y}}_t}.
  \label{eq:chain-rule}
\end{equation}
Substituting \eqref{eq:polar-flow} and using the nuclear-norm identity of the
polar direction gives
\begin{equation}
  -\dot L_t=\eta_t^X\Nnorm{\mathbf{G}_t\mathbf{Y}_t^\top}+\eta_t^Y\Nnorm{\mathbf{X}_t^\top \mathbf{G}_t}\ge0.
  \label{eq:descent-identity}
\end{equation}
Applying \Cref{lem:nuclear-product} and the step-size condition to the first
term,
\begin{equation}
  \eta_t^X\Nnorm{\mathbf{G}_t\mathbf{Y}_t^\top}
  \le\frac{\eta_0}{\Snorm{\mathbf{Y}_t}}\Fnorm{\mathbf{G}_t}\Fnorm{\mathbf{Y}_t}
  =\eta_0\Fnorm{\mathbf{G}_t}\sqrt{\erank(\mathbf{Y}_t)}.
  \label{eq:bound-x}
\end{equation}
Similarly,
\begin{equation}
  \eta_t^Y\Nnorm{\mathbf{X}_t^\top \mathbf{G}_t}\le\eta_0\Fnorm{\mathbf{G}_t}\sqrt{\erank(\mathbf{X}_t)}.
  \label{eq:bound-y}
\end{equation}
Combining \eqref{eq:descent-identity}--\eqref{eq:bound-y} and then applying
\Cref{lem:ce-gradient},
\begin{equation}
  -\dot L_t\le\sqrt2\,\eta_0L_t
  \left(\sqrt{\erank(\mathbf{X}_t)}+\sqrt{\erank(\mathbf{Y}_t)}\right).
  \label{eq:gronwall}
\end{equation}
For finite matrix entries, \eqref{eq:objective} satisfies $L_t>0$, so $\log L_t$
is well defined. Dividing \eqref{eq:gronwall} by $L_t$ and integrating from $0$
to $T$,
\begin{equation*}
  \log\frac{L_T}{L_0}
  \ge-\sqrt2\,\eta_0\int_0^T
  \left(\sqrt{\erank(\mathbf{X}_s)}+\sqrt{\erank(\mathbf{Y}_s)}\right)ds.
\end{equation*}
Exponentiating gives the bound stated in \Cref{thm:mechanistic}.

\section{Proof of the Endpoint Loss Lower Bound}
\label{app:endpoint}

\subsection{A Linear Lower Bound on the Cross-Entropy Loss}
\label{app:linear-bound}

For $d\ge2$ and an index $c\in[d]$, define the unit vector
\begin{equation}
  \mathbf{v}_c:=\sqrt{\frac{d-1}{d}}\mathbf{e}_c-\frac{1}{\sqrt{d(d-1)}}\sum_{j\ne c}\mathbf{e}_j,
  \qquad
  \|\mathbf{v}_c\|_2=1.
  \label{eq:vc}
\end{equation}
Fix a single column $\mathbf{w}\in\Rbb^d$, whose loss is
\begin{equation}
  \ell_c(\mathbf{w})=\log\!\left(1+\sum_{j\ne c}e^{w_j-w_c}\right).
  \label{eq:column-loss}
\end{equation}
By the AM--GM inequality,
\begin{equation}
  \sum_{j\ne c}e^{w_j-w_c}
  \ge(d-1)\exp\!\left[-\left(w_c-\frac1{d-1}\sum_{j\ne c}w_j\right)\right].
  \label{eq:amgm}
\end{equation}
A direct computation gives
\begin{equation}
  w_c-\frac1{d-1}\sum_{j\ne c}w_j=\sqrt{\frac d{d-1}}\ip{\mathbf{v}_c}{\mathbf{w}}.
  \label{eq:margin-identity}
\end{equation}
For any $R>0$, write
\begin{equation}
  \beta_{d,R}
  :=R\sqrt{\frac d{d-1}}
  \frac{(d-1)e^{-R\sqrt{d/(d-1)}}}{1+(d-1)e^{-R\sqrt{d/(d-1)}}}>0,
  \label{eq:beta}
\end{equation}
and let $f(t)=\log(1+(d-1)e^{-t})$ with $t_0=R\sqrt{d/(d-1)}$. Since $f$ is
convex, its tangent at $t_0$ is a global lower bound. Combining this with
\eqref{eq:amgm}--\eqref{eq:margin-identity},
\begin{equation}
  \begin{aligned}
    \ell_c(\mathbf{w})
    &\ge f\!\left(\sqrt{\frac d{d-1}}\ip{\mathbf{v}_c}{\mathbf{w}}\right)\\
    &\ge f(t_0)+\frac{\beta_{d,R}}{R}\left(R-\ip{\mathbf{v}_c}{\mathbf{w}}\right)\\
    &\ge\beta_{d,R}-\frac{\beta_{d,R}}{R}\ip{\mathbf{v}_c}{\mathbf{w}},
  \end{aligned}
  \label{eq:linear-bound}
\end{equation}
where the last step drops the nonnegative term $f(t_0)$. The bound holds for
every $\mathbf{w}$, with $R$ selecting the tangent point.

\subsection{Proof of \texorpdfstring{\Cref{thm:endpoint}}{Theorem 2.2}}
\label{app:endpoint-proof}

For each target index $c$, write
\begin{equation}
  S_c=\{i:c_i=c\},
  \qquad
  m_c=|S_c|,
  \label{eq:groups}
\end{equation}
and let $\mathbf{1}_{S_c}\in\Rbb^n$ be the indicator vector of $S_c$. Summing
\eqref{eq:linear-bound} over $i\in S_c$ and applying the Cauchy--Schwarz
inequality,
\begin{equation}
  \begin{aligned}
    \sum_{i\in S_c}\ell_c(\mathbf{w}_i)
    &\ge\beta_{d,R}m_c
      -\frac{\beta_{d,R}}{R}\ip{\mathbf{v}_c}{\sum_{i\in S_c}\mathbf{w}_i}\\
    &\ge\beta_{d,R}\left(m_c-\frac{\Snorm{\mathbf{W}}}{R}\sqrt{m_c}\right),
  \end{aligned}
  \label{eq:group-bound}
\end{equation}
where we used
\begin{equation*}
  \left\|\sum_{i\in S_c}\mathbf{w}_i\right\|_2
  =\Snorm{\mathbf{W}\mathbf{1}_{S_c}}\le\Snorm{\mathbf{W}}\sqrt{m_c}.
\end{equation*}
Summing over all $c$,
\begin{equation}
  L(\mathbf{W})\ge\beta_{d,R}
  \left[n-\frac{\Snorm{\mathbf{W}}}{R}\sum_{c=1}^{d}\sqrt{m_c}\right].
  \label{eq:sum-bound}
\end{equation}
Now take $R=\Fnorm{\mathbf{W}}/\sqrt n$ and write $C_d(\mathbf{W}):=\beta_{d,R}$. Under the
bounded-scale assumption in \Cref{thm:endpoint}, $R\le\Snorm{\mathbf W}\le B=O(1)$,
so the tangent scale used in the bound remains in the prescribed $O(1)$ regime. Since only $K$ of
the $m_c$ are nonzero, the Cauchy--Schwarz inequality gives
\begin{equation}
  \sum_{c=1}^{d}\sqrt{m_c}\le\sqrt{K\sum_c m_c}=\sqrt{Kn},
  \label{eq:cs-groups}
\end{equation}
while
\begin{equation}
  \frac{\Snorm{\mathbf{W}}}{R}=\frac{\sqrt n\Snorm{\mathbf{W}}}{\Fnorm{\mathbf{W}}}
  =\sqrt{\frac n{\erank(\mathbf{W})}}.
  \label{eq:ratio-erank}
\end{equation}
Substituting \eqref{eq:cs-groups}--\eqref{eq:ratio-erank} into
\eqref{eq:sum-bound} yields
\begin{equation}
  L(\mathbf{W})\ge C_d(\mathbf{W})\,n\left(1-\sqrt{\frac{K}{\erank(\mathbf{W})}}\right),
  \label{eq:product-bound-app}
\end{equation}
For the square setting of \Cref{thm:endpoint}, set $d=n=k$; this is exactly
the claimed endpoint bound. In the unit-scale setting relevant to the
discussion, $\Snorm{\mathbf W}=1$ gives
$R=\Fnorm{\mathbf W}/\sqrt{k}=\sqrt{\erank(\mathbf W)/k}\in[1/\sqrt{k},1]$.
Since $R\mapsto\beta_{k,R}$ is continuous and strictly positive on this
interval, $\underline C_k:=\min_{R\in[1/\sqrt{k},1]}\beta_{k,R}>0$ and
$C_k(\mathbf W)\ge\underline C_k$. Consequently, for any $\delta>0$, the
regime $\erank(\mathbf W)\ge K+\delta$ has the uniform lower bound
\[
L(\mathbf W)\ge \underline C_k\,k
\left(1-\sqrt{\frac{K}{K+\delta}}\right)>0.
\]

\subsection{How the Factor Effective Ranks Control the Product}
\label{app:factor-product}

\begin{lemma}[Two-factor product]
\label{lem:factor-product}
For any nonzero $\mathbf{X}\in\Rbb^{d\times k}$ and $\mathbf{Y}\in\Rbb^{k\times n}$ with $\mathbf{X}\mathbf{Y}\ne0$,
\begin{equation}
  \erank(\mathbf{X}\mathbf{Y})\ge\erank(\mathbf{X})+\erank(\mathbf{Y})-k.
  \label{eq:factor-product-app}
\end{equation}
\end{lemma}

\begin{proof}
Let $\mathbf{A}=\mathbf{X}^\top \mathbf{X}$ and $\mathbf{M}=\mathbf{Y}\mathbf{Y}^\top$, normalized as
\begin{equation*}
  \widehat{\mathbf{A}}=\frac{\mathbf{A}}{\Snorm{\mathbf{X}}^2},
  \qquad
  \widehat{\mathbf{M}}=\frac{\mathbf{M}}{\Snorm{\mathbf{Y}}^2}.
\end{equation*}
Then $0\preceq\widehat{\mathbf{A}},\widehat{\mathbf{M}}\preceq \mathbf{I}$, and
\begin{equation*}
  \tr(\widehat{\mathbf{A}})=\erank(\mathbf{X}),
  \qquad
  \tr(\widehat{\mathbf{M}})=\erank(\mathbf{Y}).
\end{equation*}
Both $\mathbf{I}-\widehat{\mathbf{A}}$ and $\mathbf{I}-\widehat{\mathbf{M}}$ are positive semidefinite, and the trace
of a product of two positive semidefinite matrices is nonnegative, so
\begin{equation}
  0\le\tr\!\left[(\mathbf{I}-\widehat{\mathbf{A}})(\mathbf{I}-\widehat{\mathbf{M}})\right]
  =k-\erank(\mathbf{X})-\erank(\mathbf{Y})+\tr(\widehat{\mathbf{A}}\widehat{\mathbf{M}}).
  \label{eq:psd-trace}
\end{equation}
On the other hand,
\begin{equation*}
  \Fnorm{\mathbf{X}\mathbf{Y}}^2=\tr(\mathbf{A}\mathbf{M})=\Snorm{\mathbf{X}}^2\Snorm{\mathbf{Y}}^2\tr(\widehat{\mathbf{A}}\widehat{\mathbf{M}}),
\end{equation*}
and $\Snorm{\mathbf{X}\mathbf{Y}}^2\le\Snorm{\mathbf{X}}^2\Snorm{\mathbf{Y}}^2$. Therefore
\begin{equation*}
  \erank(\mathbf{X}\mathbf{Y})=\frac{\Fnorm{\mathbf{X}\mathbf{Y}}^2}{\Snorm{\mathbf{X}\mathbf{Y}}^2}
  \ge\tr(\widehat{\mathbf{A}}\widehat{\mathbf{M}})
  \ge\erank(\mathbf{X})+\erank(\mathbf{Y})-k,
\end{equation*}
where the last step uses \eqref{eq:psd-trace}.
\end{proof}

\end{document}